\documentclass{article}

\usepackage{arxiv_preprint,times}

\usepackage{amsmath,amsfonts,bm}

\def\eqref#1{equation~\ref{#1}}

\def\1{\bm{1}}

\DeclareMathAlphabet{\mathsfit}{\encodingdefault}{\sfdefault}{m}{sl}
\SetMathAlphabet{\mathsfit}{bold}{\encodingdefault}{\sfdefault}{bx}{n}

\usepackage{amsmath,amssymb,amsthm}
\usepackage{graphicx}
\usepackage{booktabs}
\usepackage{float}
\usepackage{algorithm,algorithmic}
\usepackage[table]{xcolor}
\definecolor{gainshade}{RGB}{55,157,93}
\newif\ifreviewchanges
\ifdefined\ReviewVersion\reviewchangestrue\else\reviewchangesfalse\fi

\usepackage{microtype}
\usepackage{hyperref}
\usepackage{url}
\hypersetup{
  colorlinks=true,
  linkcolor=black,
  citecolor=black,
  urlcolor=blue
}

\newtheorem{theorem}{Theorem}
\newtheorem{proposition}[theorem]{Proposition}
\newtheorem{lemma}[theorem]{Lemma}

\theoremstyle{definition}

\newcommand{\vmf}{\operatorname{vMF}}

\title{Does a Shared Temperature Imply a Shared\\
Angular Scale in Probabilistic\\
Contrastive Learning?}

\author{\parbox{\dimexpr\textwidth-2\tabcolsep\relax}{\centering
Ningkang Peng\textsuperscript{1}, Qianfeng Yu\textsuperscript{1}, Jingyang Mao\textsuperscript{1}, Xiaoqian Peng\textsuperscript{2},\\
Tingyu Lu\textsuperscript{3}, Peirong Ma\textsuperscript{1}, Yanhui Gu\textsuperscript{1}\\[0.5em]
\normalfont\textsuperscript{1}Nanjing Normal University\\
\textsuperscript{2}Nanjing University of Chinese Medicine\quad\textsuperscript{3}Tohoku University\\
\texttt{nkpeng@nnu.edu.cn}, \texttt{gu@njnu.edu.cn}}}
\iclrfinalcopy

\renewcommand{\eqref}[1]{\textup{(\ref{#1})}}

\begin{document}

\maketitle

\begin{abstract}
In probabilistic contrastive learning, a shared temperature is commonly
interpreted as a shared similarity scale, but this interpretation does not hold for high-dimensional distributional class representations. We study
the exact von Mises--Fisher (vMF) probabilistic score used by ProCo when
representation dimension and class concentration grow jointly. We prove
that the score retains a class-dependent leading angular gain
$g_c=A_c/\tau$, where $A_c$ is the mean resultant length. This gain enters
Softmax competition, pairwise decision boundaries, and feature gradients.
On real CIFAR-LT, ImageNet-LT, and iNaturalist representations, the theory
accurately predicts boundary movements and local gradient changes under
the full vMF score. Classwise temperature adjustment also changes the
cosine-zero intercept and finite-dimensional response. We construct
intercept-preserving and Pure Angular controls to separate the leading
gain from these accompanying changes. Complete gain equalization yields
a shared-scale cosine prototype rule at leading order; a finite-dimensional margin condition guarantees agreement of
the two classifiers. Across 16 frozen representation settings, prediction
agreement is 98.43--99.99\%, with disagreements concentrated at small
cosine margins. In controlled contrastive-only training with the
training-frequency prior, Pure Angular editing improves both learned
representations at all tested CIFAR-10/100 imbalance factors and retains
positive changes on ImageNet-LT. Thus vMF concentration not only describes
class distributions, but also forms a decision and learning scale in
high-dimensional probabilistic contrastive learning.
\end{abstract}

\section{Introduction}

In similarity-based contrastive learning, temperature controls how strongly
a score responds to changes in similarity
\citep{chen2020simclr,khosla2020supervised,wang2021behaviour}. When all classes share one
temperature, a natural intuition is that they share the same similarity
scale. When a class is represented by a feature distribution rather than
a single prototype, however, its score aggregates similarity across the
entire distribution. \emph{Does a shared temperature still imply a shared
angular response?}
More broadly, once similarity is aggregated over a class distribution, its
effective scale can depend on the geometry of that distribution as well as
on the nominal temperature.

We make this distribution-dependent scaling effect precise through the
exact von Mises--Fisher (vMF) probabilistic score used by ProCo
\citep{du2024probabilisticcontrastivelearninglongtailed}. ProCo models each
class by a vMF distribution on the unit sphere and obtains its score by
integrating exponential similarity over that distribution. Such spherical
distributional modeling builds on directional statistics
\citep{banerjee2005clustering,Dryden_2005}. At fixed
dimension, increasing concentration makes the distribution collapse toward
its mean direction, recovering a cosine score scaled by the shared
temperature. Modern neural representations have hundreds or thousands of
dimensions, while class concentration need not grow much faster than
dimension. We therefore study their joint growth and identify the
class-dependent structure that survives in the high-dimensional score.

Our first result identifies a class-dependent leading angular gain.
For class mean resultant length $A_c$, the coefficient $g_c=A_c/\tau$
gives the leading linear angular response. We establish a uniform
expansion of the exact score as dimension and concentration grow jointly,
controlling both the score remainder and its angular derivative.
Consequently, a shared nominal temperature does not generally imply a shared
effective angular scale once classes are represented distributionally; in
the vMF case, this dependence is captured explicitly by $A_c/\tau$.

This gain propagates through multiclass learning. It changes relative
competition in Softmax, rotates pairwise boundary normals, and reweights
class directions in feature gradients. We derive these consequences from
the same expansion and test them on real CIFAR-LT, ImageNet-LT, and
iNaturalist representations. Accurate predictions of full-score boundary
movements and local feature gradients turn $g_c$ from an analytical
coefficient into a quantitatively testable decision and learning scale.

A further question is how to intervene on this gain in isolation.
Changing class temperatures modifies $A_c/\tau_c$, but also changes the
cosine-zero intercept and finite-dimensional response. We distinguish
direct temperature adjustment, intercept-preserving adjustment, and Pure
Angular editing of only the leading linear term. These realizations can
share the same target gain while producing different complete scores.
Fixed-state experiments show that restoring the intercept or editing only
the angular term removes anomalous predictions caused by large
temperature-induced intercept drift.

The decomposition also reveals the classification geometry after complete
equalization. The Pure Angular score is a
shared-scale cosine prototype score plus a finite-dimensional correction.
We give a sufficient condition involving the top-two cosine margin and
the correction range for the two rules to predict the same class.
Across 16 frozen CIFAR, ImageNet-LT, and iNaturalist representation settings,
prediction agreement is 98.43--99.99\%; remaining disagreements concentrate
at small cosine margins.

Finally, we test whether the local score mechanism propagates to learned
representations. Following the distinction between representation and
classifier learning in long-tailed recognition \citep{kang2020decoupling},
we evaluate the learned features with a common linear classifier.
In contrastive-only training with the training-frequency
prior, complete Pure Angular equalization improves Encoder and Projection
readouts at every tested CIFAR-10/100 imbalance factor (10, 50, and 100),
with positive changes on ImageNet-LT. These results connect the
high-dimensional score to decision geometry, feature gradients, and
representations learned through controlled training.

Our contributions are:
\begin{itemize}
\item \textbf{High-dimensional angular gain.} We establish a uniform expansion
of the exact vMF score and its feature derivative under joint dimension--
concentration growth, identifying the leading gain $g_c=A_c/\tau$.
\item \textbf{Decision and learning mechanisms.} We derive how gain enters
Softmax competition, pairwise boundaries, and feature gradients, and
quantitatively validate these predictions on real representations.
\item \textbf{Gain realizations and cosine geometry.} We separate three
score realizations and characterize complete equalization through a
finite-dimensional margin condition for agreement with cosine prototypes.
\end{itemize}

\section{High-Dimensional vMF Angular Gain}
\label{sec:gain-foundation}

\subsection{Exact vMF score in ProCo}
For $p\ge3$, let $z,\mu_c\in\mathbb S^{p-1}$ be a normalized query and a class mean
direction. A class feature $X_c\sim\vmf(\mu_c,\kappa_c)$ has density
$\exp(\kappa_c\mu_c^\top x)/Z_p(\kappa_c)$, where
$Z_p(\kappa)=(2\pi)^{p/2}I_{p/2-1}(\kappa)/\kappa^{p/2-1}$.
Writing $t=1/\tau$, $\rho=\mu_c^\top z$, and $\nu=p/2-1$, the exact
log-moment-generating score in ProCo is
\begin{equation}
 q_c(\rho;\tau)=\log\mathbb E\exp(tz^\top X_c)
 =\log\frac{Z_p(\widetilde\kappa_c)}{Z_p(\kappa_c)},\qquad
 \widetilde\kappa_c=\sqrt{\kappa_c^2+2\kappa_ct\rho+t^2}.
 \label{eq:main-exact}
\end{equation}
The mean resultant length is
$A_c=I_{\nu+1}(\kappa_c)/I_\nu(\kappa_c)$, so that
$\mathbb E X_c=A_c\mu_c$. For a class offset $b_c$, the contrastive loss
uses logits $s_c=b_c+q_c$ and $L_y=-\log\operatorname{softmax}(s)_y$.
An empirical prior corresponds to $b_c=\gamma\log\pi_c$.
Here $\mu_c$ describes the class mean direction and $A_c$ its mean
directional concentration. We ask whether the angular response of the
exact score is determined by the shared temperature alone at high dimension.

\subsection{From the fixed-dimensional limit to the high-dimensional regime}
At fixed $p$ and bounded $t$, taking $\kappa_c\to\infty$ gives
$q_c(\rho;\tau)=t\rho+O(\kappa_c^{-1})$: the usual shared-temperature
cosine limit. This limit is not uniform when dimension grows with
concentration. The classical large-order Bessel expansion
\citep{Olver_1954} instead gives, for $\kappa_c=\nu\lambda_c$,
\begin{equation}
 A_c=r_0(\lambda_c)+O(\nu^{-1}),\qquad
 r_0(\lambda)=\frac{\lambda}{1+\sqrt{1+\lambda^2}}.
 \label{eq:main-ratio}
\end{equation}
When concentration and dimension remain of the same order, $A_c$ generally
does not approach one. The shared temperature therefore retains a
class-dependent angular response. We make this statement precise below.
\begin{theorem}[Centered angular response]
\label{thm:centered-gain}
Fix $0<\lambda_-\le\lambda_c\le\lambda_+<\infty$ and
$0\le t\le T<\infty$. Uniformly over $\rho\in[-1,1]$ as $\nu\to\infty$,
\begin{equation}
 \boxed{q_c(\rho;\tau)=q_c(0;\tau)+g_c\rho+r_c(\rho;\tau),
 \qquad g_c=\frac{A_c}{\tau},\quad r_c(0;\tau)=0,}
 \label{eq:main-decomposition}
\end{equation}
where $\sup_\rho(|r_c|+|\partial_\rho r_c|)=O(\nu^{-1})$.
The cosine-zero intercept also satisfies $q_c(0;\tau)=O(\nu^{-1})$.
\end{theorem}
The full proof appears in Appendix~\ref{app:centered-proof}; the second-order statement is
Theorem~\ref{thm:joint}. The finite-dimensional coefficient $A_c/\tau$
and the asymptotic coefficient $r_0(\lambda_c)/\tau$ differ by
$O(\nu^{-1})$. Neither is asserted to be the exact slope at every cosine.
Indeed, the exact derivative is
\begin{equation}
 \partial_\rho q_c(\rho;\tau)
 =\frac{\kappa_c t}{\widetilde\kappa_c}
   \frac{I_{\nu+1}(\widetilde\kappa_c)}{I_\nu(\widetilde\kappa_c)}.
 \label{eq:main-exact-slope}
\end{equation}
The proof uses the exact derivative and a uniform Bessel-ratio expansion,
controlling the remainder and its first derivative together rather than
differentiating a value-only error bound. The full proof and second-order
joint expansion are in the appendix.

\subsection{Consequences for decision and learning}
Put $a_c=b_c+q_c(0;\tau)$ and $\ell_c=a_c+g_c\mu_c^\top z$.
For a fixed number of classes, the theorem gives
$P_c=P_c^{\rm lin}+O(\nu^{-1})$, where
$P^{\rm lin}=\operatorname{softmax}(\ell)$, and
$L_y=-\log P_y^{\rm lin}+O(\nu^{-1})$.
The same decomposition gives the exact pairwise log odds
\begin{equation}
 \log\frac{P_a}{P_b}
 =a_a-a_b+(g_a\mu_a-g_b\mu_b)^\top z+r_a-r_b.
 \label{eq:main-boundary}
\end{equation}
The leading boundary is
\begin{equation}
 (a_a-a_b)+(g_a\mu_a-g_b\mu_b)^\top z=0.
 \label{eq:main-leading-boundary}
\end{equation}
Gain rotates the linear boundary normal; prior and intrinsic intercepts
change its offset. With equal offsets, the boundary on the prototype short
arc solves $g_a\cos u=g_b\cos(\theta-u)$, where
$\theta=\arccos(\mu_a^\top\mu_b)$. Equal gains place this root at
$u=\theta/2$.

Let $\Pi_z=I-zz^\top$ project onto the sphere's tangent space. The exact
and linear-mean loss gradients are connected by
\begin{equation}
 \nabla_{\mathbb S}L_y
 =\Pi_z\sum_c(P_c-\mathbf1_{c=y})
       [g_c+\partial_\rho r_c]\mu_c
 =\Pi_z\sum_c(P_c^{\rm lin}-\mathbf1_{c=y})g_c\mu_c
       +O(\nu^{-1}).
 \label{eq:main-gradient}
\end{equation}
Gain affects learning through two paths: it directly rescales each class
direction in the gradient and changes the Softmax probabilities that weight
all competing directions. Class-dependent angular gain therefore determines
the leading geometry of the high-dimensional score and forms both a
decision scale and a local learning scale.

\section{Controlled Studies of Angular Gain}
\label{sec:controlled-studies}

We test whether the high-dimensional mechanism can be observed, predicted,
and controlled in real models. We examine gain heterogeneity, full-score
predictions, alternative gain realizations, and learned representations.
Tables~\ref{tab:frozen-vmf} and~\ref{tab:angular-training} report its effects
on frozen decisions and learned features, respectively.

\begin{table}[t]
\centering\small
\setlength{\tabcolsep}{3pt}
\caption{\textbf{Representation accuracy after contrastive-only training.}
Long-tail linear-probe accuracy (\%) after feature learning with the training-frequency
prior. All interventions use $\beta=0$. Results report mean $\pm$ sample
standard deviation over three seeds. $\Delta$ is the difference between
the displayed Pure Angular and ProCo means in percentage points; darker
green marks a larger increase.
Readout protocols are given in Section~\ref{sec:new-learning} and
Appendices~\ref{sec:representation-evaluation-protocol} and~\ref{sec:algorithm-hyperparameters}.}
\label{tab:angular-training}
\begin{tabular}{@{}lcccccr@{}}
\toprule
Dataset / IF & Feature & ProCo & \shortstack{Direct\\temperature\\$\beta=0$} & \shortstack{Intercept-\\preserving\\$\beta=0$} & \shortstack{Pure Angular\\$\beta=0$} & $\Delta$\\
\midrule
CIFAR-10 / 10 & Encoder & $86.28 \pm 0.09$ & $86.52 \pm 0.15$ & $86.54 \pm 0.11$ & $86.93 \pm 0.10$ & \cellcolor{gainshade!35.0}$+0.65$\\
 & Projection & $87.10 \pm 0.06$ & $87.55 \pm 0.08$ & $87.65 \pm 0.16$ & $87.96 \pm 0.13$ & \cellcolor{gainshade!43.1}$+0.86$\\
CIFAR-10 / 50 & Encoder & $75.73 \pm 0.12$ & $75.93 \pm 0.19$ & $76.02 \pm 0.10$ & $76.14 \pm 0.17$ & \cellcolor{gainshade!25.8}$+0.41$\\
 & Projection & $77.66 \pm 0.07$ & $77.64 \pm 0.11$ & $78.46 \pm 0.14$ & $78.83 \pm 0.09$ & \cellcolor{gainshade!55.0}$+1.17$\\
CIFAR-10 / 100 & Encoder & $68.24 \pm 0.16$ & $67.89 \pm 0.13$ & $68.13 \pm 0.13$ & $68.43 \pm 0.11$ & \cellcolor{gainshade!17.3}$+0.19$\\
 & Projection & $72.65 \pm 0.10$ & $72.48 \pm 0.14$ & $72.74 \pm 0.08$ & $73.12 \pm 0.16$ & \cellcolor{gainshade!28.1}$+0.47$\\
CIFAR-100 / 10 & Encoder & $53.08 \pm 0.12$ & $53.11 \pm 0.11$ & $53.24 \pm 0.16$ & $53.82 \pm 0.14$ & \cellcolor{gainshade!38.5}$+0.74$\\
 & Projection & $52.89 \pm 0.09$ & $53.24 \pm 0.15$ & $53.09 \pm 0.18$ & $53.75 \pm 0.11$ & \cellcolor{gainshade!43.1}$+0.86$\\
CIFAR-100 / 50 & Encoder & $37.97 \pm 0.15$ & $39.14 \pm 0.17$ & $38.46 \pm 0.17$ & $38.91 \pm 0.13$ & \cellcolor{gainshade!46.2}$+0.94$\\
 & Projection & $37.40 \pm 0.08$ & $37.44 \pm 0.16$ & $38.14 \pm 0.11$ & $37.76 \pm 0.12$ & \cellcolor{gainshade!23.8}$+0.36$\\
CIFAR-100 / 100 & Encoder & $34.62 \pm 0.14$ & $34.11 \pm 0.19$ & $34.32 \pm 0.16$ & $35.39 \pm 0.20$ & \cellcolor{gainshade!39.6}$+0.77$\\
 & Projection & $34.19 \pm 0.10$ & $33.68 \pm 0.17$ & $33.27 \pm 0.12$ & $34.41 \pm 0.09$ & \cellcolor{gainshade!18.5}$+0.22$\\
ImageNet-LT & Encoder & $23.05 \pm 0.12$ & $22.93 \pm 0.15$ & $23.12 \pm 0.16$ & $23.22 \pm 0.19$ & \cellcolor{gainshade!16.5}$+0.17$\\
 & Projection & $36.86 \pm 0.06$ & $37.05 \pm 0.10$ & $37.04 \pm 0.13$ & $37.07 \pm 0.11$ & \cellcolor{gainshade!18.1}$+0.21$\\
\bottomrule
\end{tabular}
\end{table}

\subsection{Does class-dependent gain exist in real representations?}

We fit classwise vMF statistics to frozen ProCo projection features from
CIFAR-100-LT, ImageNet-LT, and iNaturalist 2018
\citep{krizhevsky2009learning,deng2009imagenet,liu2019largescale,vanhorn2018inaturalist}.
At the shared temperature $\tau_0=0.1$, the resulting $A_c/\tau_0$ varies
across classes. Figure~\ref{fig:gain-discovery} displays all-class
heterogeneity together with a real class-pair example. These results confirm that a shared temperature leaves substantial
class-dependent angular-scale heterogeneity in real ProCo representations.

\begin{figure}[t]
\centering
\includegraphics[width=139.7mm]{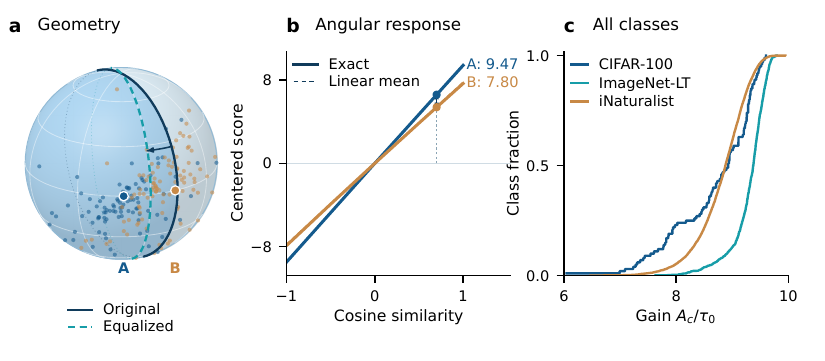}
\caption{\textbf{Class-dependent angular gain at a shared temperature.}
(a) A training-selected CIFAR-100-LT pair in a boundary-sign-preserving
spherical projection. Solid and dashed curves show original and equalized
linear-mean boundaries at zero prior; shading shows the original regions.
(b) Exact centered scores and linear-mean approximations for the same pair.
(c) All-class empirical distributions of $A_c/\tau_0$ in three projection
spaces, with equal weight per class.}
\label{fig:gain-discovery}
\end{figure}

\subsection{Can high-dimensional theory predict the full vMF score?}

\paragraph{Numerical regime check.}
Independent high-precision Bessel evaluations confirm the predicted
residual orders. Successive joint approximations have median slopes
$-0.999,-2.000,-2.999$ (Figure~\ref{fig:joint-validation}), and the
second-corrected score reaches median absolute error $2.32\times10^{-7}$
in the joint regime. The fixed-dimensional expansion is more accurate
at low dimension and high concentration, consistent with its distinct limit.

\begin{figure}[t]
\centering
\includegraphics[width=112mm]{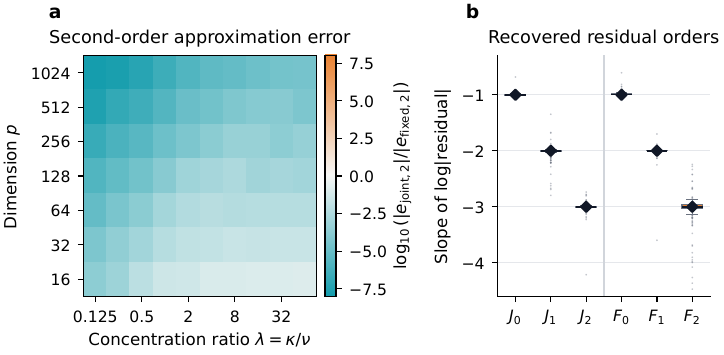}
\caption{\textbf{Accuracy and residual orders of the joint expansion.}
Left: median log error ratio of second-order joint and fixed-dimensional
approximations over the prespecified temperature--cosine grid.
Right: observed residual orders for successive approximations; diamonds
mark theoretical orders, boxes show interquartile ranges, and whiskers use
$1.5$ interquartile ranges. Formulas, validity domains, and fitting details
are in Theorem~\ref{thm:joint} and Appendix~\ref{app:modern-protocol}.}
\label{fig:joint-validation}\label{fig:phase-domain}
\end{figure}

\paragraph{Real boundaries and gradients.}
We freeze CIFAR-100-LT, ImageNet-LT, and iNaturalist projection features
with dimensions 128, 1,024, and 1,024. On each dataset, 60 pairs are
preselected from training geometry: five highly similar and five random
pairs in each of six shot-group pair types. Every pair is retained,
including pairs without a root on the prototype short arc.
The intervention changes the finite-dimensional target gain from
$g_c$ to the class mean $\bar g$, implemented by $\tau_c=A_c/\bar g$.
Both the exact score and the joint leading approximation are evaluated at
these same temperatures, with the same prototypes and prior.
Figure~\ref{fig:main-boundaries} compares predicted and exact movements.
At zero prior, their displacement directions agree on every pair with a
unique root before and after intervention. The remaining 7, 8, and 11
pairs acquire a root after equalization; none has an original short-arc
root. Root creation is recorded separately from displacement error.
We additionally evaluate all $\binom{10}{2}=45$ unordered class pairs from
an original contrastive-only CIFAR-10-LT IF100 model's frozen 128-dimensional
Projection features. This exhaustive extension uses the same temperature
intervention, zero evaluation prior, and exact/B$_0$ score comparison.
All 45 pairs have unique boundaries before and after intervention, and all
45 movement directions agree; the displacement MAE is $0.0116^\circ$.
Pair inclusion is independent of the measured prediction error.

\begin{figure}[t]
\centering
\includegraphics[width=139.7mm]{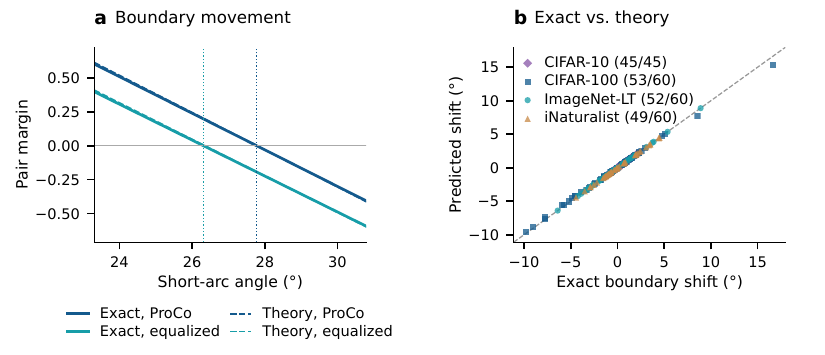}
\caption{\textbf{Predicted and exact boundary displacement.}
Left: exact and joint-leading margins near the roots of a fixed
CIFAR-100-LT pair. Right: theory versus exact displacement for all
preselected pairs with a unique root in both conditions at zero prior.
The legend gives comparable-pair coverage. The CIFAR-10 exhaustive extension has MAE $0.0116^\circ$ (45/45 pairs); for CIFAR-100, ImageNet-LT, and iNaturalist, displacement MAEs are
$0.119^\circ$, $0.0027^\circ$, and $0.0042^\circ$, respectively.}
\label{fig:main-boundaries}\label{fig:boundary-closure}
\end{figure}

On real queries, the linear-mean approximation predicts the direction of
exact margin changes with $98.63\%$, $99.92\%$, and $100\%$ agreement on
CIFAR-100-LT, ImageNet-LT, and iNaturalist, respectively. Leading tangent
gradients also closely track exact gradients at fixed class statistics
(Appendix~\ref{sec:gradient-protocol}). Together, these results connect the
score expansion to observable decision changes and local learning signals.

\subsection{Can angular gain be isolated from other temperature effects?}
\label{sec:realization}

For positive fitted resultants, let $\bar g=C^{-1}\sum_cg_c$ and
$g_c^\star=\bar g+\beta(g_c-\bar g)$. The original ProCo
coefficient is recovered at $\beta=1$ and complete linear-mean equalization
at $\beta=0$. We compare three realizations with the original score
$q_c^0(\rho)=q_c(\rho;\tau_0)$, using $\tau_c^\star=A_c/g_c^\star$:
\begin{align}
 q_c^{\rm temp}(\rho)&=q_c(\rho;\tau_c^\star),
 \label{eq:temperature-control}\\
 q_c^{\rm preserve}(\rho)&=q_c(\rho;\tau_c^\star)
  -q_c(0;\tau_c^\star)+q_c(0;\tau_0),
 \label{eq:intercept-control}\\
 q_c^{\rm angular}(\rho)&=q_c^0(\rho)+(g_c^\star-g_c)\rho.
 \label{eq:Pure Angular-control}
\end{align}
For $\beta\ne1$, direct temperature adjustment generally changes the intercept, linear coefficient,
and nonlinear remainder in~\eqref{eq:main-decomposition}.
Intercept preservation restores the original cosine-zero score but leaves
the new remainder. Pure Angular editing preserves both the original
intercept and remainder, changing only the stated linear coefficient.
It does not force all exact derivatives in~\eqref{eq:main-exact-slope}
to coincide.

Writing $h_c^0=q_c(0;\tau_0)$, $h_c^\star=q_c(0;\tau_c^\star)$,
and $r_c^0,r_c^\star$ for the corresponding remainders, the exact differences are
\begin{equation}
 q_c^{\rm temp}-q_c^{\rm preserve}=h_c^\star-h_c^0,
 \qquad q_c^{\rm preserve}-q_c^{\rm angular}=r_c^\star-r_c^0.
 \label{eq:realization-differences}
\end{equation}
These identities separately identify temperature-induced intercept changes
and finite-dimensional remainder changes, rather than combining them into
a single temperature effect.

\paragraph{Complete equalization and cosine prototypes.}
At $\beta=0$, define $\delta_c(z)=q_c^0(\rho_c)-g_c\rho_c=h_c^0+r_c^0(\rho_c)$.
Then $q_c^{\rm angular}=\bar g\rho_c+\delta_c$ exactly.
Let $m_{\cos}(z)=\rho_{(1)}-\rho_{(2)}$ be the largest-to-second-largest
cosine gap and $D(z)=\max_c\delta_c-\min_c\delta_c$. 
\begin{equation}
 \bar g\,m_{\cos}(z)>D(z)
 \quad\Longrightarrow\quad
 \arg\max_c q_c^{\rm angular}(z)=\arg\max_c\rho_c.
 \label{eq:cosine-agreement}
\end{equation}
Indeed, no residual difference can reverse the leading winner's margin.
Under Theorem~\ref{thm:centered-gain} and $\bar g$ bounded away from zero,
disagreement is confined to an $O(\nu^{-1})$ cosine-margin band.

\paragraph{Frozen classification and the cosine control.}
We compare original vMF, cosine prototypes, and complete Pure Angular
scores on the same frozen features and fitted class statistics
(Table~\ref{tab:frozen-vmf}). Encoder and Projection features are normalized
for spherical scoring, with zero prior offset throughout.

Cosine prototypes closely track Pure Angular equalization across all sixteen
conditions: the accuracy difference is $-0.170$ to $+0.110$ points
(Pure Angular minus cosine), with 98.43--99.99\% prediction agreement.
The overwhelming majority of disagreements occur in the lowest
cosine-margin region.
CIFAR-10 IF50/100 provides the complementary case: equalization decreases
Encoder/Projection accuracy by $0.75/0.14$ and $0.74/0.23$ points,
respectively, while retaining 99.46--99.95\% agreement with cosine.
The shared-scale cosine connection therefore holds empirically across both
improvements and decreases relative to the original vMF classifier.
Appendix~\ref{app:cosine-equivalence} retains all three interventions and
per-query checks.

\paragraph{A diagnostic failure state.}
In a CIFAR-100 IF10 uniform-prior state, direct temperature equalization
collapses predictions to one class and raises cross-entropy from $2.25$ to $72.34$.
The affected class's cosine-zero score rises from $0.389$ to $79.146$.
Restoring the original intercept reduces cross-entropy to $2.15$ and
restores prediction diversity. Pure Angular editing gives essentially the
same recovery (Figure~\ref{fig:realization}). Temperature-induced intercept and finite-dimensional response changes
can therefore substantially alter the complete score, while Pure Angular
editing isolates its leading angular component.

\begin{figure}[t]
\centering
\includegraphics[width=120mm]{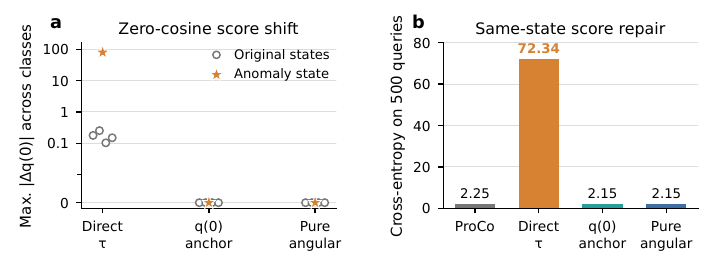}
\caption{\textbf{Intercept drift under matched gain interventions.}
(a) Maximum absolute classwise change in the cosine-zero score
(linear scale below 0.01, logarithmic above).
Open circles denote four separately fixed original ProCo states; stars
denote interventions on the independently trained IF10 uniform-prior
temperature-failure state.
(b) Exact cross-entropy on the same 500 queries from the failure state.
Within each state, queries, statistics, and prior are fixed; the three
interventions share $g_c^*=\bar g$.}
\label{fig:realization}
\end{figure}

\begin{table}[t]
\centering\small
\setlength{\tabcolsep}{5pt}
\caption{\textbf{Frozen classification after complete gain equalization.}
Accuracy and cosine--Pure Angular prediction agreement (\%) on identical
features and class statistics, with zero evaluation offset and $\beta=0$.
CIFAR and ImageNet use original contrastive-only models; iNaturalist uses
the official ProCo model. Full realization comparisons are in Appendix~\ref{app:cosine-equivalence};
evaluation protocols are given in Appendices~\ref{sec:controlled-protocols}
and~\ref{sec:frozen-numerical-protocols}.}
\label{tab:frozen-vmf}
\begin{tabular}{@{}lccccr@{}}
\toprule
Dataset / IF & Feature & Original vMF & Cosine & Pure Angular & Agreement\\
\midrule
CIFAR-10 / 10 & Encoder & 85.21 & 85.82 & 85.90 & 99.74\\
 & Projection & 87.85 & 87.84 & 87.83 & 99.99\\
CIFAR-10 / 50 & Encoder & 81.47 & 80.76 & 80.72 & 99.46\\
 & Projection & 80.72 & 80.58 & 80.58 & 99.95\\
CIFAR-10 / 100 & Encoder & 78.76 & 78.19 & 78.02 & 99.51\\
 & Projection & 77.81 & 77.59 & 77.58 & 99.93\\
CIFAR-100 / 10 & Encoder & 48.87 & 52.48 & 52.50 & 98.76\\
 & Projection & 54.38 & 55.06 & 55.15 & 99.32\\
CIFAR-100 / 50 & Encoder & 39.96 & 44.24 & 44.35 & 98.43\\
 & Projection & 42.90 & 43.57 & 43.63 & 99.00\\
CIFAR-100 / 100 & Encoder & 36.77 & 41.53 & 41.59 & 98.52\\
 & Projection & 39.80 & 40.84 & 40.90 & 98.65\\
ImageNet-LT & Encoder & 41.39 & 45.98 & 45.97 & 99.70\\
 & Projection & 46.14 & 46.39 & 46.39 & 99.88\\
iNaturalist 2018 & Encoder & 57.61 & 64.75 & 64.76 & 99.89\\
 & Projection & 63.07 & 65.40 & 65.42 & 99.84\\
\bottomrule
\end{tabular}
\end{table}

On CIFAR-100, ImageNet-LT, and iNaturalist,
local cross-entropy gradients are also close:
the stacked-query relative difference between Pure Angular and cosine is
1.45--1.62\% on CIFAR-100 and 0.12--0.15\% on ImageNet-LT/iNaturalist,
with median gradient cosines above 0.99995. Complete results and finite-dimensional bounds are in
Appendix~\ref{app:cosine-equivalence}.

\subsection{Does angular gain affect learned representations?}
\label{sec:new-learning}

The frozen experiments show that class-dependent angular gain changes the
full vMF score, decision boundaries, and local feature gradients. We now
test whether this local mechanism propagates to representations learned
through training.

The main experiments use an objective containing only the probabilistic
contrastive loss. Within each long-tail setting, original ProCo and the
alternative angular-gain realizations use the same data and optimization
configuration. After training, we perform linear evaluation separately on
Encoder and Projection representations. Table~\ref{tab:angular-training}
reports means and sample standard deviations over three seeds for all four
methods. Reported improvements compare the displayed Pure Angular and
ProCo means.

With the training-frequency prior, complete Pure Angular equalization
improves both Encoder and Projection representations at all CIFAR-10 and
CIFAR-100 imbalance factors: IF10, IF50, and IF100
(Table~\ref{tab:angular-training}). Improvements range from $+0.19$ to
$+1.17$ percentage points on CIFAR-10 and from $+0.22$ to $+0.94$ points
on CIFAR-100. All six long-tail settings and twelve representation
readouts exhibit positive changes.

These results show that controlled changes to the leading angular gain
influence feature learning beyond their immediate effect on fixed scores.
Direct temperature adjustment and intercept preservation do not exhibit
the same consistent improvement pattern as Pure Angular, further showing
that the complete score realization affects the resulting representations.

On ImageNet-LT, original ProCo achieves $23.05\pm0.12\%$ and
$36.86\pm0.06\%$ Encoder/Projection accuracy, while complete Pure Angular
equalization achieves $23.22\pm0.19\%$ and $37.07\pm0.11\%$, respectively,
corresponding to mean changes of $+0.17/+0.21$ percentage points.
The single-run experiment amplifying between-class gain differences to
$\beta=1.5$ yields $23.56/37.76\%$. Angular gain therefore remains an
effective variable for representation learning at scale, and the best
gain structure need not correspond to complete equalization.

Together, the frozen-score, boundary, local-gradient, and end-to-end
training results establish a continuous evidence chain:
\[
\boxed{\text{Angular gain}\,\longrightarrow\,\text{Score response}
\,\longrightarrow\,\text{Boundaries and gradients}
\,\longrightarrow\,\text{Learned representations}.}
\]

\section{Discussion and Scope}

\subsection{Class frequency and concentration}
The class mean resultant length $A_c$ describes directional geometry in a
given representation space, whereas the class count $n_c$ describes the
amount of statistical support. They play different roles in the model.

For independent unit-vector samples with population mean $A\mu$, the
sample mean vector length $R_n$ satisfies
\[
\mathbb E[R_n^2]=A^2+\frac{1-A^2}{n}.
\]
Finite sample size therefore affects the estimation of concentration
statistics without changing their population definition. Similarly, when
the current query is included in the class statistics used for scoring,
its leading-order self-inclusion effect is
\[
\Delta_{\mathrm{self}}=\frac{1-A^2}{\tau n}.
\]
These finite-sample effects can be more pronounced in long-tail classes
and affect practical probabilistic scores together with the class prior.
We examine these statistical effects through support subsampling and
frozen replay, respectively; complete results appear in the appendix
(Section~\ref{sec:finite-sample-diagnostics}).

\subsection{The vMF score model and real representation distributions}
Our theory directly analyzes the exact probabilistic score at a given vMF
statistical state. For fixed $(\mu_c,\kappa_c)$, its score expansion,
decision boundaries, and local gradients can be computed and tested directly.

Real neural representations need not follow a single vMF distribution
exactly. Multimodal class structure, feature normalization, and dynamically
updated online statistics can all lead to departures from that assumption.
Two levels must therefore be distinguished: the score, boundary, and local
gradient mechanisms are deterministic results at a given statistical
state, whereas population-error and random-sampling conclusions require
additional distributional assumptions.

End-to-end experiments further show that these local score mechanisms can
propagate to the final learned representation. We do not interpret the
local gradient expansion as a closed-form description of the complete
training trajectory.

\subsection{How should angular gain be intervened on?}
We identify the leading angular coefficient
\[
g_c=\frac{A_c}{\tau},
\]
but its different realizations have different effects on the complete
probabilistic score.

Direct classwise temperature adjustment changes the leading angular term,
cosine-zero intercept, and finite-dimensional remainder together.
Intercept preservation removes the intercept change while retaining the
new finite-dimensional response. Pure Angular edits only the linear
angular term identified by the theory, preserving the original intercept
and remainder.

After complete equalization, the leading Pure Angular score reduces to a
shared-scale cosine prototype classifier. The 98.43--99.99\% prediction
agreement on real frozen representations further indicates that
class-dependent gain is one of the main sources of departure from
shared-scale cosine geometry in the original vMF score. The remaining
finite-dimensional corrections primarily affect low-margin regions.

This result suggests a broader principle: interventions on a theoretical
variable in a probabilistic score should distinguish the intended change
from the other score structures altered by its parameterization.
Pure Angular provides a controlled intervention aligned with the
high-dimensional decomposition, allowing angular gain itself to be
studied separately.

\section{Conclusion}
We study class-dependent angular scales in high-dimensional vMF
probabilistic contrastive scores. When representation dimension and class
concentration grow jointly, the exact score retains the leading angular
gain $g_c=A_c/\tau$. A shared temperature therefore does not imply a
shared angular response.

This gain enters Softmax competition, pairwise decision boundaries, and
feature gradients. Experiments on real CIFAR-LT, ImageNet-LT, and
iNaturalist representations show that the high-dimensional theory
accurately predicts boundary movements and local gradient changes under
the full vMF score.

We distinguish direct temperature adjustment, intercept preservation,
and Pure Angular as three gain realizations. Complete equalization of
the leading gain brings the probabilistic classifier toward a shared-scale
cosine prototype rule, whose finite-dimensional departures are characterized
by the cosine margin and residual score terms. Controlled training further
shows that changes to this angular scale propagate to learned representations.

More broadly, these results show that in distributional representation
learning, a global temperature need not define a global similarity scale:
the effective response can also depend on class-specific distributional
geometry. In the vMF case, this interaction is captured explicitly by
$A_c/\tau$.

\subsection*{AI use statement}
Generative AI tools assisted with code development and debugging,
theoretical derivations and proof drafting, literature searches and reference
organization, and manuscript drafting, editing, and language polishing.
The authors are responsible for the final content, including the correctness
of the code, mathematical claims, and reported results.

\subsection*{Reproducibility statement}
The appendix provides derivations, dataset and estimator protocols, fixed
intervention settings, checkpoint selection, and readout details. The research
artifacts retain complete grids, class-pair selections, training trajectories,
fixed inputs, and source data for the figures.

\bibliographystyle{arxiv_references}
\bibliography{references}

\clearpage
\appendix
\numberwithin{equation}{section}
\raggedbottom
\begingroup
\begin{center}
{\Large\bfseries Does a Shared Temperature Imply a Shared\\
Angular Scale in Probabilistic Contrastive Learning?\par}
\medskip
{\Large Supplementary Material\par}
\end{center}

\medskip
{\large\bfseries Contents\par}

\definecolor{contentsblue}{RGB}{47,112,171}
\hypersetup{linkcolor=contentsblue}
\setcounter{tocdepth}{2}

\makeatletter
\renewcommand*\l@section[2]{%
  \vspace{4pt}%
  \@dottedtocline{1}{0em}{1.5em}{\bfseries #1}{\bfseries #2}}
\renewcommand*\l@subsection[2]{%
  \@dottedtocline{2}{1.5em}{2.3em}{#1}{#2}}

\small
\newif\ifshowappendixcontents
\showappendixcontentsfalse

\let\originalcontentsline\contentsline
\def\appendixstart{appendix.A}
\def\appendixstartalt{section.A}

\renewcommand\contentsline[4]{%
  \def\currentanchor{#4}%
  \ifx\currentanchor\appendixstart
    \showappendixcontentstrue
  \fi
  \ifx\currentanchor\appendixstartalt
    \showappendixcontentstrue
  \fi
  \ifshowappendixcontents
    \originalcontentsline{#1}{#2}{#3}{#4}%
  \fi}

\@starttoc{toc}
\makeatother
\endgroup

\clearpage
\section{Related Work}
\label{nav:related_work:1}
\label{sec:related-work}

\paragraph{Contrastive learning under class imbalance.}
Instance-based frameworks such as SimCLR and MoCo learn representations by
contrasting augmented observations \citep{chen2020simclr,he2020moco};
supervised contrastive learning also treats same-label examples as positives
\citep{khosla2020supervised}.
Long-tailed extensions address the unequal influence of classes through
learnable class centers in PaCo \citep{cui2021parametric}, class averaging and
class completion in Balanced Contrastive Learning \citep{zhu2022balanced},
and uniformly distributed targets in Targeted Supervised Contrastive
Learning \citep{li2022targeted}.
These methods motivate studying how class geometry enters contrastive
competition. Our question concerns the angular response already induced by
integrating similarity over a class distribution, and how to isolate this
response within the exact probabilistic score.

\paragraph{Probabilistic contrastive learning.}
ProCo introduces the vMF-based probabilistic contrastive formulation studied
in this work, where classwise scores are obtained from ratios of vMF
normalizers rather than from point-prototype cosine similarities
\citep{du2024probabilisticcontrastivelearninglongtailed}.
Related formulations reuse or extend this distributional scoring structure in
long-tailed recognition and out-of-distribution settings.
PATT adopts the same normalization-ratio form for long-tailed OOD detection
\citep{he2024longtailedoutofdistributiondetectionprioritizing}, while BAPE
uses an explicit vMF generative Bayes classifier with class priors
\citep{du2025bapelearningexplicitbayes}.
VMF-GOS combines closely related vMF score terms with virtual-outlier
construction for long-tailed OOD detection
\citep{peng2026vmfgosgeometryguidedvirtualoutlier}.
MARVEL derives an explicit $O(\kappa^{-1})$ correction in a fixed-dimensional,
large-concentration regime
\citep{anudeep2026marvelmarginawarerobustvon}.
Our analysis instead studies the exact ProCo log-MGF score when dimension and
concentration grow jointly, and follows the resulting class-dependent angular
response through multiclass competition, decision geometry, and learning
gradients.

\paragraph{Class priors, temperature, and contrastive geometry.}
Class-prior adjustment and angular scaling act on different components of a
classifier.
Related long-tail methods rebalance losses using the effective number of
samples \citep{cui2019classbalanced}, introduce label-distribution-aware
margins \citep{cao2019ldam}, or account for label-distribution shift through
Balanced Softmax \citep{ren2020balanced}.
Frequency-based logit adjustment modifies additive class offsets
\citep{menon2021longtail}, whereas the concentration-dependent factor
$A_c/\tau$ studied here changes the angular response of the probabilistic
score.
This distinction is important in long-tailed settings, where class frequency
can simultaneously affect the chosen prior and the statistical uncertainty of
estimated class geometry.

Temperature has also been studied as a global control of contrastive hardness
and sample weighting
\citep{wang2021behaviour}, while alignment and uniformity characterize the
geometry induced by contrastive representation learning on the hypersphere
\citep{wang2020alignment}.
These perspectives concern a shared temperature or global representation
geometry.
In contrast, the vMF score introduces an additional class-dependent angular
scale through the class resultant length, even when the nominal temperature is
shared.

\paragraph{Prototype geometry and representation evaluation.}
Prototypical Networks classify using distances to class prototypes
\citep{snell2017prototypical}, while proxy-based metric learning uses learned
representative points to construct a training objective
\citep{movshovitz2017proxies}.
CosFace and ArcFace explicitly modify cosine or angular margins in normalized
classifiers \citep{wang2018cosface,deng2019arcface}.
Our cosine-equivalence result characterizes when an edited distributional
score agrees with a shared-scale cosine prototype rule, including the
finite-dimensional residual that can change low-margin predictions.

Decoupling representation learning from classifier learning is an established
approach to long-tailed recognition \citep{kang2020decoupling}.
This distinction is also essential for interpreting our experiments:
frozen vMF rescoring tests a decision rule on unchanged features, whereas a
common linear evaluator tests features produced by different training
objectives. Improvements in these two evaluations support different claims.

\paragraph{Directional statistics and high-dimensional asymptotics.}
Mixtures of vMF distributions provide a classical model for directional
clustering and connect distributional modeling to cosine-based clustering
\citep{banerjee2005clustering}.
Our analysis builds on classical results for directional distributions and
modified Bessel functions, including large-order asymptotic expansions
\citep{Olver_1954}, computation and bounds for adjacent Bessel ratios
\citep{Amos_1974,Hornik_2013,Ruiz_Antolin_2016}, and vMF parameter estimation
\citep{Sra_2011}.
High-dimensional directional statistics further characterize concentration
phenomena on growing-dimensional spheres
\citep{Dryden_2005}.

Using these tools, we derive a uniform joint dimension--concentration expansion
of the exact probabilistic contrastive score and its first derivative.
This lets the same high-dimensional quantity be propagated from the classwise
score to Softmax probabilities, pairwise decision boundaries, and feature
gradients.
The resulting analysis complements classical directional asymptotics by
identifying their operational consequences inside a modern probabilistic
contrastive objective.
\section{The Exact Probabilistic Contrastive Score}
\label{sec:exact-score}

Let
\[
p\ge 3,
\qquad
\nu=\frac p2-1,
\]
and let $\sigma$ denote surface measure on
$\mathbb S^{p-1}$.

For a natural parameter
\[
\eta\in\mathbb R^p,
\]
define the spherical partition function and its log-partition function by
\[
Z_p(\eta)
=
\int_{\mathbb S^{p-1}}
\exp(\eta^\top x)\,d\sigma(x),
\qquad
\Phi_p(\eta)
=
\log Z_p(\eta).
\]

For class $j$, write
\[
\eta_j
=
\kappa_j\mu_j,
\qquad
\|\mu_j\|=1,
\]
and let the normalized query satisfy
\[
\|z\|=1.
\]
Define
\[
\rho_j
=
\mu_j^\top z,
\qquad
t
=
\tau^{-1}.
\]
After exponential tilting by the query,
\[
\widetilde\kappa_j
=
\|\kappa_j\mu_j+tz\|
=
\sqrt{
\kappa_j^2
+
2\kappa_j t\rho_j
+
t^2
}.
\]

The exact class score is
\[
q_j(z)
=
\Phi_p(\eta_j+tz)
-
\Phi_p(\eta_j),
\]
or equivalently
\begin{equation}
q_j(z)
=
\log
\mathbb E_{X_j\sim\mathrm{vMF}_p(\mu_j,\kappa_j)}
\exp(tz^\top X_j).
\label{eq:exact-logmgf}
\end{equation}

Thus the ProCo score is a log moment-generating function of the directional
similarity $z^\top X_j$.

\subsection{Exponential-Family and Bregman Representation}

Define the Bregman divergence generated by $\Phi_p$ as
\[
D_{\Phi_p}(a,b)
=
\Phi_p(a)
-
\Phi_p(b)
-
\nabla\Phi_p(b)^\top(a-b).
\]

\begin{proposition}[Exponential-family and Bregman view]
\label{prop:bregman-view}
Let
\[
R_\nu(x)
=
\frac{I_{\nu+1}(x)}{I_\nu(x)}.
\]
Then
\begin{equation}
q_j
=
t\rho_j R_\nu(\kappa_j)
+
D_{\Phi_p}(\eta_j+tz,\eta_j),
\label{eq:bregman-score}
\end{equation}
with
\[
D_{\Phi_p}(\eta_j+tz,\eta_j)\ge 0.
\]
\end{proposition}

\begin{proof}
For the vMF exponential family,
\[
\nabla\Phi_p(\eta_j)
=
\mathbb E[X_j].
\]
Rotational symmetry gives
\[
\mathbb E[X_j]
=
R_\nu(\kappa_j)\mu_j.
\]
Adding and subtracting the first-order Taylor term of
$\Phi_p(\eta_j+tz)$ around $\eta_j$ yields
\[
q_j
=
tz^\top\nabla\Phi_p(\eta_j)
+
D_{\Phi_p}(\eta_j+tz,\eta_j),
\]
which becomes Eq.~\eqref{eq:bregman-score}.
Nonnegativity follows from convexity of the log-partition function.
\end{proof}

The decomposition separates the projected class mean from higher-order
distributional corrections:
\[
t\rho_jR_\nu(\kappa_j)
=
t\,z^\top\mathbb E[X_j].
\]
The remaining Bregman term contains covariance and higher-order fluctuation
effects.

More explicitly,
\begin{equation}
D_{\Phi_p}(\eta_j+tz,\eta_j)
=
\int_0^t
(t-u)\,
\operatorname{Var}_{\eta_j+uz}
(z^\top X_j)
\,du.
\label{eq:bregman-variance}
\end{equation}

\subsection{Endpoint and Integral Representations}

Define
\[
\phi_\nu(x)
=
\log I_\nu(x)
-
\nu\log x.
\]
Since
\[
\phi_\nu'(x)
=
R_\nu(x),
\]
the score admits the equivalent endpoint representation
\[
q_j
=
\phi_\nu(\widetilde\kappa_j)
-
\phi_\nu(\kappa_j),
\]
and hence the integral representation
\begin{equation}
q_j
=
\int_{\kappa_j}^{\widetilde\kappa_j}
R_\nu(x)\,dx.
\label{eq:score-integral}
\end{equation}

These forms are useful for the two asymptotic regimes developed in
Appendix~\ref{sec:asymptotics}.

\subsection{Exact Derivatives}

The log-partition form also gives exact derivatives without approximation.
Along the temperature-tilt path,
\[
\partial_t q_j
=
z^\top
\mathbb E_t[X_j],
\]
and
\[
\partial_t^2 q_j
=
\operatorname{Var}_t(z^\top X_j),
\]
where $\mathbb E_t$ and $\operatorname{Var}_t$ are taken under the tilted
natural parameter
\[
\eta_j+tz.
\]

The exact ambient feature gradient is
\begin{equation}
\nabla_z q_j
=
t\,\mathbb E_t[X_j].
\label{eq:exact-score-gradient}
\end{equation}

Using the vMF mean under the tilted state gives
\begin{equation}
\nabla_z q_j
=
tR_\nu(\widetilde\kappa_j)
\frac{\kappa_j\mu_j+tz}
{\widetilde\kappa_j}.
\label{eq:exact-vector-gradient}
\end{equation}

Since
\[
\rho_j=\mu_j^\top z,
\]
the exact derivative with respect to cosine similarity is
\begin{equation}
\partial_{\rho_j}q_j
=
\frac{\kappa_j t}
{\widetilde\kappa_j}
R_\nu(\widetilde\kappa_j).
\label{eq:exact-rho-derivative}
\end{equation}

These exact derivatives are used to establish the
$C^1$ high-dimensional control in Theorem~1 and Appendix~\ref{sec:centered-proof}.

\subsection{Multiclass Logits and Loss}

For a class offset $b_j$, define
\[
s_j
=
b_j+q_j.
\]
The multiclass probability and per-example cross-entropy are
\begin{equation}
P_j
=
\frac{e^{s_j}}
{\sum_k e^{s_k}},
\qquad
L_y
=
-s_y
+
\log\sum_j e^{s_j}.
\label{eq:multiclass-loss}
\end{equation}

When the model uses a frequency-based prior offset, we write
\[
b_j
=
\gamma\log\pi_j.
\]

We distinguish three class quantities throughout the analysis:
\[
n_j,
\qquad
\kappa_j,
\qquad
A_j
=
R_\nu(\kappa_j).
\]
Here $n_j$ is statistical support, $\kappa_j$ is the fitted concentration,
and $A_j$ is the mean resultant length.
Their finite-sample relationships are analyzed separately in
Appendix~\ref{sec:finite-sample}.
\section{High-Dimensional Asymptotic Expansions}
\label{sec:asymptotics}

The fixed-dimensional large-concentration limit and the joint
dimension--concentration limit lead to different angular scales.
This section makes the two regimes explicit and gives the higher-order
joint expansion used throughout the paper.

\subsection{Fixed Dimension and Large Concentration}
\label{sec:fixed-dimensional-expansion}

We begin with the classical fixed-dimensional regime.
Let $p$, $t=\tau^{-1}$, and $\rho$ be fixed, and let
$\kappa\to\infty$.

\begin{proposition}[Fixed-dimensional score]
\label{prop:fixed-dimensional-score}
For fixed $p$, $t$, and $\rho$,
\[
q(\rho;\tau)
=
t\rho
+
\frac{Q_1}{\kappa}
+
\frac{Q_2}{\kappa^2}
+
O(\kappa^{-3}),
\]
where
\[
Q_1
=
\frac{t^2(1-\rho^2)}{2}
-
\frac{(p-1)t\rho}{2},
\]
and
\[
Q_2
=
\frac{(p-1)(p-3)t\rho}{8}
+
\frac{(p-1)t^2(2\rho^2-1)}{4}
-
\frac{t^3\rho(1-\rho^2)}{2}.
\]
\end{proposition}

The leading term is therefore
\[
q(\rho;\tau)
=
\frac{\rho}{\tau}
+
O(\kappa^{-1}),
\]
which is the familiar shared-temperature cosine limit.

At $t=1$, the first correction overlaps the fixed-dimensional
large-concentration result considered in MARVEL
\citep{anudeep2026marvelmarginawarerobustvon}.
The important point for the present work is that the coefficients
$Q_1,Q_2,\ldots$ depend explicitly on $p$.
Consequently, this expansion is not ordered when $p$ grows at the
same rate as $\kappa$.

\subsection{Joint Dimension--Concentration Growth}
\label{sec:joint-expansion}

We now let dimension and concentration grow together.
Write
\[
\kappa=\nu\lambda,
\qquad
\nu=\frac p2-1,
\]
with
\[
0<\lambda_- \le \lambda \le \lambda_+<\infty.
\]

Define
\[
S=\sqrt{1+\lambda^2},
\]
and
\[
r_0(\lambda)
=
\frac{\lambda}{1+S},
\]
\[
r_1(\lambda)
=
-\frac{\lambda}{2S^2},
\]
\[
r_2(\lambda)
=
\frac{\lambda(4-\lambda^2)}{8S^5}.
\]

The leading coefficient $r_0$ is the limiting mean resultant length:
\[
R_\nu(\nu\lambda)
=
r_0(\lambda)
+
O(\nu^{-1}).
\]

\begin{theorem}[Uniform joint score expansion]
\label{thm:joint-score-expansion}\label{thm:joint}
Fix
\[
0<\lambda_-<\lambda_+<\infty,
\qquad
0\le t\le T<\infty.
\]
Uniformly for
\[
\lambda\in[\lambda_-,\lambda_+],
\qquad
\rho\in[-1,1],
\]
the exact vMF log-MGF score satisfies
\[
q
=
B_0
+
\frac{B_1}{\nu}
+
\frac{B_2}{\nu^2}
+
O(\nu^{-3}),
\]
where
\[
B_0
=
t\rho r_0,
\]
\[
B_1
=
t\rho r_1
+
\frac{t^2}{2}
\left[
(1-\rho^2)\frac{r_0}{\lambda}
+
\rho^2 r_0'
\right],
\]
and
\[
\begin{aligned}
B_2
={}&
t\rho r_2
+
\frac{t^2}{2}
\left[
(1-\rho^2)\frac{r_1}{\lambda}
+
\rho^2 r_1'
\right]
\\
&+
\frac{t^3}{6}
\left[
\rho^3 r_0''
+
3\rho(1-\rho^2)\frac{r_0'}{\lambda}
-
3\rho(1-\rho^2)\frac{r_0}{\lambda^2}
\right].
\end{aligned}
\]
The same expansion is $C^1$-uniform on the sphere for fixed class
statistics, with derivatives taken in the tangent space.
\end{theorem}

\begin{proof}[Proof of Theorem~\ref{thm:joint-score-expansion}]
Set $\varepsilon=\nu^{-1}$ and $d=\widetilde\kappa-\kappa$.
The reverse triangle inequality gives $|d|\le t\le T$.
For sufficiently large $\nu$, every point between $\lambda$ and
$\lambda+\varepsilon d$ lies in a fixed compact subset of $(0,\infty)$.
Lemma~\ref{lem:bessel-ratio-remainder}, proved independently in
Appendix~\ref{sec:additional-proofs}, gives
\[
R_\nu(\nu u)=a_\varepsilon(u)+e_\nu(u),\qquad
 a_\varepsilon=r_0+\varepsilon r_1+\varepsilon^2r_2,
 \qquad |e_\nu(u)|\le\varepsilon^3.
\]
The exact integral representation in Appendix~\ref{sec:exact-score} implies
\[
q=\int_0^d R_\nu(\kappa+s)\,ds
 =Q_\nu+E_\nu,\qquad
Q_\nu=\int_0^d a_\varepsilon(\lambda+\varepsilon s)\,ds,
\qquad |E_\nu|\le T\varepsilon^3.
\]
These are oriented integrals when $d<0$.
Expanding the square root defining $\widetilde\kappa$ yields
\[
d=t\rho+\varepsilon\frac{t^2(1-\rho^2)}{2\lambda}
-\varepsilon^2\frac{t^3\rho(1-\rho^2)}{2\lambda^2}
+O_{C^1(\rho)}(\varepsilon^3).
\]
The compact lower bound on $\lambda$ makes this Taylor remainder uniform
also for $\rho=\pm1$ and $t=0$. Taylor expansion of the smooth integrand gives
\[
Q_\nu=(r_0+\varepsilon r_1+\varepsilon^2r_2)d
 +\frac{\varepsilon}{2}(r_0'+\varepsilon r_1')d^2
 +\frac{\varepsilon^2}{6}r_0''d^3
 +O_{C^1(\rho)}(\varepsilon^3).
\]
Indeed, the omitted integrand and its endpoint contribution are bounded
by $C\varepsilon^3$ because $|s|\le T$ and the required derivatives of
$r_0,r_1,r_2$ are bounded on that compact interval.
Substitution of the expansion of $d$ and collection of powers of
$\varepsilon$ give exactly $B_0,B_1,B_2$ in the theorem.

It remains to control the derivative of the exact remainder, rather than
formally differentiate a scalar big-$O$ bound. Since $\rho$ enters the
integral only through its endpoint,
\[
\partial_\rho E_\nu
 =e_\nu(\lambda+\varepsilon d)\,\partial_\rho d,
\qquad
\partial_\rho d=\frac{\kappa t}{\widetilde\kappa}.
\]
For $\kappa\ge2T$, $|\partial_\rho d|\le2T$, so
$|\partial_\rho E_\nu|\le2T\varepsilon^3$.
Finally, for any scalar function $f(\rho)$,
$\nabla_{\mathbb S}f=f'(\rho)(I-zz^\top)\mu$ and
$\|(I-zz^\top)\mu\|\le1$. Thus both the score error and its spherical
gradient are uniformly $O(\nu^{-3})$, proving the stated $C^1$ claim.
\end{proof}

The leading score is
\[
B_0
=
\frac{r_0(\lambda)}{\tau}\rho.
\]
Since
\[
A
=
R_\nu(\kappa)
=
r_0(\lambda)
+
O(\nu^{-1}),
\]
we may equivalently write
\[
B_0
=
\frac{A}{\tau}\rho
+
O(\nu^{-1}).
\]

Thus the joint regime retains the class-dependent angular coefficient
\[
g
=
\frac{A}{\tau},
\]
or asymptotically
\[
g
\sim
\frac{r_0(\lambda)}{\tau}.
\]

Only when
\[
\lambda\to\infty
\]
does
\[
r_0(\lambda)\to 1,
\]
recovering the fixed-dimensional cosine scale $1/\tau$.

\subsection{Why the Two Limits Differ}
\label{sec:noncommuting-limits}

The fixed-dimensional and joint limits correspond to different
asymptotic paths.

In the fixed-dimensional regime,
\[
p=\text{constant},
\qquad
\kappa\to\infty,
\]
and the vMF distribution collapses toward its mean direction.
The leading score therefore becomes
\[
q(\rho;\tau)\sim \frac{\rho}{\tau}.
\]

In the joint regime,
\[
\kappa=\nu\lambda,
\qquad
\nu\to\infty,
\]
and a finite value of $\lambda$ leaves a nontrivial resultant
$r_0(\lambda)<1$.
Dimension-dependent contributions that appear at successive orders in the
fixed-$p$ expansion are resummed into
\[
r_0,\qquad r_1,\qquad r_2,\ldots
\]
in the joint expansion.

This is the origin of the class-dependent angular scale studied in the
main text.

\subsection{Independent Derivation from Directional Cumulants}
\label{sec:cumulant-derivation}

The same coefficients can be recovered without expanding the score
endpoint directly.

Let
\[
Y=z^\top X,
\qquad
R=R_\nu(\kappa).
\]
Directional derivatives of the vMF log-partition function give
\[
\mathbb E[Y]
=
\rho R,
\]
\[
\operatorname{Var}(Y)
=
(1-\rho^2)\frac{R}{\kappa}
+
\rho^2 R',
\]
and
\[
\operatorname{cum}_3(Y)
=
\rho^3 R''
+
3\rho(1-\rho^2)
\left(
\frac{R'}{\kappa}
-
\frac{R}{\kappa^2}
\right).
\]

Hence
\[
q
=
t\mathbb E[Y]
+
\frac{t^2}{2}\operatorname{Var}(Y)
+
\frac{t^3}{6}\operatorname{cum}_3(Y)
+
O(\nu^{-3}).
\]

Substituting
\[
R_\nu(\nu\lambda)
=
r_0
+
\frac{r_1}{\nu}
+
\frac{r_2}{\nu^2}
+
O(\nu^{-3})
\]
recovers exactly the coefficients
$B_0,B_1,B_2$ above.

This representation also gives an interpretation of the expansion:
the leading term comes from the projected class mean,
the first correction from directional covariance,
and the second correction includes directional skewness.

\subsection{Range of the Expansion}
\label{sec:expansion-range}

Theorem~\ref{thm:joint-score-expansion} assumes bounded inverse temperature
$t$.
More generally, the natural small parameter is $t/\nu$.

For
\[
t=o(\sqrt{\nu}),
\]
the leading linear-mean approximation can still have vanishing absolute
error under the same joint scaling.
When
\[
t=\Theta(\nu),
\]
the bounded-$t$ polynomial expansion is no longer appropriate and a
resummed free-energy limit is required.

All theoretical and empirical uses of the joint expansion in this paper
remain within the stated regime.
\section{Proof of the Centered Angular-Response Theorem}
\label{app:centered-proof}
\label{sec:centered-proof}

We prove Theorem~1 from the main text.
Recall
\[
R_\nu(x)
=
\frac{I_{\nu+1}(x)}{I_\nu(x)},
\qquad
A_c
=
R_\nu(\kappa_c),
\]
and let
\[
\kappa_c=\nu\lambda_c,
\qquad
0<\lambda_-\le \lambda_c\le \lambda_+<\infty.
\]
Throughout this section, the inverse temperature
\[
t=\tau^{-1}
\]
is bounded by a fixed constant $T$, and all bounds are uniform over
\[
\rho\in[-1,1].
\]

The exact score derivative from Eq.~\eqref{eq:exact-rho-derivative} is
\[
\partial_\rho q_c(\rho;\tau)
=
\frac{\kappa_c t}{\widetilde\kappa_c}
R_\nu(\widetilde\kappa_c),
\qquad
\widetilde\kappa_c
=
\sqrt{
\kappa_c^2
+
2\kappa_c t\rho
+
t^2
}.
\]

We first control the displacement of the tilted concentration.
Since $\kappa_c=\nu\lambda_c$ and $t$ is bounded,
\[
\widetilde\kappa_c-\kappa_c
=
t\rho+O(\nu^{-1}),
\]
uniformly in $\rho$. Hence
\[
\frac{\kappa_c}{\widetilde\kappa_c}
=
1+O(\nu^{-1}).
\]

Next, the uniform large-order ratio expansion gives
\[
R_\nu(\nu\lambda)
=
r_0(\lambda)
+
O(\nu^{-1}),
\qquad
r_0(\lambda)
=
\frac{\lambda}
{1+\sqrt{1+\lambda^2}},
\]
uniformly for $\lambda$ in compact positive intervals.

To control the derivative of the ratio without differentiating a
value-only asymptotic remainder, we use the exact Riccati identity
\[
R_\nu'(x)
=
1
-
R_\nu(x)^2
-
\frac{2\nu+1}{x}R_\nu(x).
\]
Substituting the uniform ratio expansion and using
\[
1-r_0(\lambda)^2-\frac{2}{\lambda}r_0(\lambda)=0
\]
gives
\[
R_\nu'(\nu\lambda)
=
O(\nu^{-1})
\]
uniformly on a slightly enlarged compact interval containing all
values between $\kappa_c/\nu$ and $\widetilde\kappa_c/\nu$.

Therefore,
\begin{equation}
R_\nu(\widetilde\kappa_c)
=
R_\nu(\kappa_c)
+
O(\nu^{-1}),
\label{eq:bessel-local-control}
\end{equation}
and the exact score derivative satisfies
\[
\partial_\rho q_c(\rho;\tau)
=
tA_c
+
O(\nu^{-1}).
\]

Define the centered remainder
\[
r_c(\rho;\tau)
=
q_c(\rho;\tau)
-
q_c(0;\tau)
-
tA_c\rho.
\]
Then
\[
r_c(0;\tau)=0
\]
and
\[
\partial_\rho r_c(\rho;\tau)
=
\partial_\rho q_c(\rho;\tau)
-
tA_c
=
O(\nu^{-1})
\]
uniformly over $\rho\in[-1,1]$.

Integrating from $0$ to $\rho$ gives
\[
r_c(\rho;\tau)
=
\int_0^\rho
\left[
\partial_v q_c(v;\tau)
-
tA_c
\right]\,dv.
\]
Because the integration interval has length at most one,
\[
\sup_{\rho\in[-1,1]}
|r_c(\rho;\tau)|
=
O(\nu^{-1}),
\]
and therefore
\[
\sup_{\rho\in[-1,1]}
\left(
|r_c(\rho;\tau)|
+
|\partial_\rho r_c(\rho;\tau)|
\right)
=
O(\nu^{-1}).
\]

It remains to control the cosine-zero intercept.
At $\rho=0$,
\[
\widetilde\kappa_c
=
\sqrt{\kappa_c^2+t^2}
=
\kappa_c
+
O(\nu^{-1}).
\]
Using the integral representation
\[
q_c(0;\tau)
=
\int_{\kappa_c}^{\widetilde\kappa_c}
R_\nu(x)\,dx
\]
and the bound
\[
0<R_\nu(x)<1,
\]
we obtain
\[
q_c(0;\tau)
=
O(\nu^{-1}).
\]

Finally, since
\[
tA_c
=
\frac{A_c}{\tau},
\]
the centered score can be written as
\[
q_c(\rho;\tau)
=
q_c(0;\tau)
+
g_c\rho
+
r_c(\rho;\tau),
\qquad
g_c
=
\frac{A_c}{\tau},
\]
with
\[
r_c(0;\tau)=0,
\qquad
\sup_\rho
\left(
|r_c|
+
|\partial_\rho r_c|
\right)
=
O(\nu^{-1}),
\]
and
\[
q_c(0;\tau)=O(\nu^{-1}).
\]
This proves Theorem~1.
\section{Softmax, Decision Geometry, and Pairwise Boundaries}
\label{sec:softmax-geometry}

The class-dependent angular gain derived in the main text affects not only
individual class scores but also multiclass competition and decision
boundaries.
This section records the corresponding deterministic geometry.

\subsection{Leading Multiclass Logits}

For the finite-dimensional centered linearization, define
\[
a_c=b_c+q_c(0;\tau),\qquad \ell_c=a_c+g_c\mu_c^\top z.
\]
At leading joint order, $q_c(0;\tau)=O(\nu^{-1})$ and
$g_c=r_0(\lambda_c)/\tau+O(\nu^{-1})$. Thus the leading logit becomes
\[
\ell_c^0
=
b_c
+
g_c^0\mu_c^\top z,
\qquad
g_c^0
=
\frac{r_0(\lambda_c)}{\tau}.
\]
The corresponding probabilities are
\[
P_c^0
=
\frac{\exp(\ell_c^0)}
{\sum_j \exp(\ell_j^0)}.
\]

From the joint expansion,
\[
q_c
=
g_c^0\mu_c^\top z
+
O(\nu^{-1}),
\]
uniformly over the compact parameter family.
For a fixed number of classes,
Softmax smoothness therefore gives
\[
P_c
=
P_c^0
+
O(\nu^{-1}),
\]
and
\[
L_y
=
L_y^0
+
O(\nu^{-1}),
\]
where
\[
L_y^0
=
-\ell_y^0
+
\log\sum_j \exp(\ell_j^0).
\]

Thus the gain $g_c^0$ changes two objects simultaneously:
the score assigned to class $c$ and the normalized probabilities that
weight all competing classes.

\subsection{Pairwise Equality Surfaces}

Consider two classes $a$ and $b$.
Their leading log-odds difference is
\begin{equation}
\log\frac{P_a^0}{P_b^0}
=
(b_a-b_b)
+
(g_a^0\mu_a-g_b^0\mu_b)^\top z.
\label{eq:leading-pairwise-logodds}
\end{equation}

Hence the leading equality surface is
\begin{equation}
(b_a-b_b)
+
(g_a^0\mu_a-g_b^0\mu_b)^\top z
=
0.
\label{eq:leading-boundary}
\end{equation}

The two terms have distinct geometric roles.

The bias difference
\[
b_a-b_b
\]
changes the offset of the boundary, whereas
\[
g_a^0\mu_a-g_b^0\mu_b
\]
determines its normal direction.

When
\[
g_a^0=g_b^0,
\]
the equality surface reduces to the shared-scale prototype rule
\[
(b_a-b_b)
+
g^0(\mu_a-\mu_b)^\top z
=
0.
\]

At equal offsets,
\[
b_a=b_b,
\]
the boundary is the spherical analogue of a prototype bisector only when
the two gains are equal.

\subsection{Boundary Geometry on the Prototype Short Arc}

Let
\[
\theta
=
\arccos(\mu_a^\top\mu_b)
\in(0,\pi)
\]
be the angle between two class prototypes.

Parameterize the short great-circle arc from class $a$ toward class $b$
by
\[
u\in[0,\theta].
\]
Along this arc,
\[
\mu_a^\top z(u)
=
\cos u,
\]
and
\[
\mu_b^\top z(u)
=
\cos(\theta-u).
\]

At zero class-offset difference, the leading pairwise margin becomes
\begin{equation}
D_0(u)
=
g_a^0\cos u
-
g_b^0\cos(\theta-u).
\label{eq:short-arc-margin}
\end{equation}

A boundary on the short arc therefore solves
\begin{equation}
g_a^0\cos u
=
g_b^0\cos(\theta-u).
\end{equation}

If
\[
g_a^0=g_b^0,
\]
then
\[
u=\frac{\theta}{2},
\]
recovering the midpoint boundary.

If
\[
g_a^0\ne g_b^0,
\]
the root generally shifts away from the midpoint.
Whenever a unique root remains on the short arc, equalizing the gains moves
the boundary toward the prototype that originally had the larger gain.
Equivalently, the class with the larger original angular gain occupies a
larger portion of the short arc.

For sufficiently extreme gain ratios, the equality root can leave the
short arc entirely.
The experiments therefore distinguish boundary displacement from boundary
creation or disappearance.

\subsection{Exact Pairwise Margins}

For finite dimension, let
\[
c
=
\mu_a^\top\mu_b.
\]
The exact class-centre margins are
\[
M_a^{\mathrm{exact}}
=
b_a-b_b
+
q_a(1)
-
q_b(c),
\]
and
\[
M_b^{\mathrm{exact}}
=
b_b-b_a
+
q_b(1)
-
q_a(c).
\]

Their leading joint-regime counterparts are
\[
M_a^0
=
b_a-b_b
+
g_a^0
-
g_b^0c,
\]
and
\[
M_b^0
=
b_b-b_a
+
g_b^0
-
g_a^0c.
\]

These expressions separate prototype geometry, class offsets, and angular
gain in a form that can be evaluated directly on frozen representations.

\subsection{Pairwise Boundaries and Class Count}

Adding additional classes changes multiclass posterior probabilities but does
not change the equality surface of a fixed pair.

Indeed, for any two classes $a$ and $b$,
\[
P_a=P_b
\quad\Longleftrightarrow\quad
s_a=s_b.
\]
The scores of all other classes cancel from this equality.

Class count does, however, affect one-versus-rest posterior thresholds.
Writing
\[
s_\star
=
\max_j s_j,
\]
we have
\[
\log\sum_j e^{s_j}
=
s_\star
+
\log N_{\mathrm{eff}},
\]
where
\begin{equation}
N_{\mathrm{eff}}
=
\sum_j
\exp(s_j-s_\star).
\end{equation}

Thus increasing the number or strength of competitors can move a
posterior threshold such as
\[
P_a=\frac12
\]
without moving the fixed pairwise equality surface
\[
s_a=s_b.
\]

This distinction is used in the controlled geometry experiments.

\subsection{Finite-Dimensional Corrections}

Using the centered decomposition
\[
q_c
=
q_c(0)
+
g_c\rho_c
+
r_c(\rho_c),
\]
the exact pairwise log odds are
\[
\log\frac{P_a}{P_b}
=
[a_a-a_b]
+
(g_a\mu_a-g_b\mu_b)^\top z
+
r_a-r_b,
\]
where
\[
a_c
=
b_c+q_c(0).
\]

The leading decision geometry is therefore perturbed by the finite-dimensional
difference
\[
r_a-r_b.
\]
The uniform bound from Theorem~1 implies
\[
r_a-r_b
=
O(\nu^{-1})
\]
under the joint high-dimensional assumptions.

This relation motivates the boundary-prediction experiments in the main text
and the more detailed root analysis in the experimental appendix.
\begin{figure}[H]
  \centering
  \includegraphics[width=.86\linewidth]{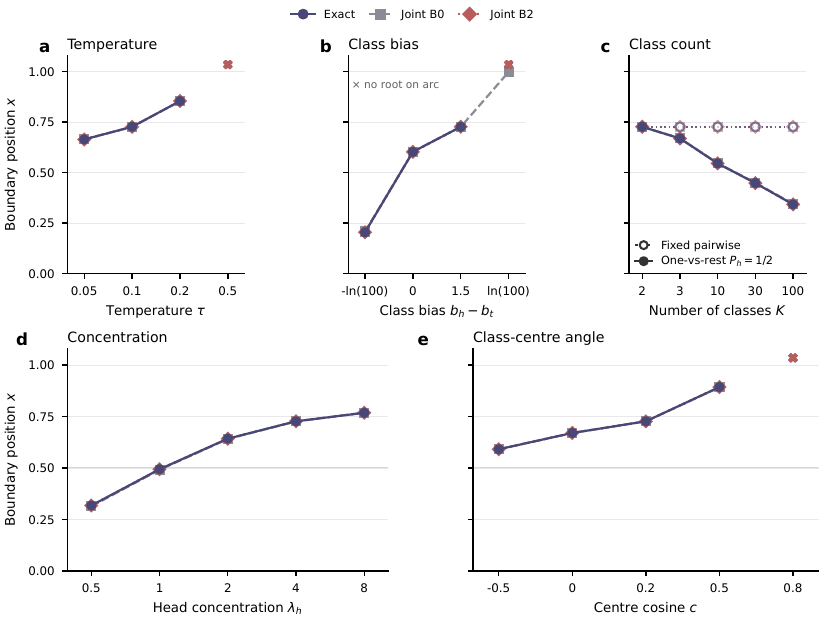}
  \caption{\textbf{Geometric effects of score parameters.} Temperature, bias, concentration, and prototype angle affect pairwise and one-versus-rest boundaries; class count affects only the latter. Joint B0/B2 are truncations through orders zero/two.}
  \label{fig:factor-app}
\end{figure}

\section{Learning Gradients and Conditional Margins}
\label{sec:risk-gradients}

The angular gain affects not only decision geometry but also the local
learning signal. We first derive the exact and leading feature gradients,
and then characterize a population-level consequence under a matched
binary vMF model.

\subsection{Exact and Leading Feature Gradients}
\label{sec:feature-gradients}

Hold the class statistics fixed while differentiating with respect to a
normalized query $z$.
For class $c$, define
\[
\eta_c=\kappa_c\mu_c,
\qquad
v_c=\eta_c+tz,
\qquad
\widetilde\kappa_c=\|v_c\|.
\]

From Appendix~\ref{sec:exact-score}, the exact score gradient is
\begin{equation}
\nabla_z q_c
=
tR_\nu(\widetilde\kappa_c)
\frac{v_c}{\widetilde\kappa_c}.
\label{eq:exact-class-gradient}
\end{equation}

For the multiclass loss
\[
L_y
=
-s_y+\log\sum_c e^{s_c},
\]
the exact ambient gradient is therefore
\begin{equation}
\nabla_z L_y
=
\sum_c
\left(
P_c-\mathbf 1[c=y]
\right)
tR_\nu(\widetilde\kappa_c)
\frac{v_c}{\widetilde\kappa_c}.
\label{eq:exact-loss-gradient}
\end{equation}

Because the representation is normalized, the relevant learning signal lies
in the tangent space of the sphere.
Let
\[
\Pi_z
=
I-zz^\top.
\]
The exact spherical gradient is
\begin{equation}
\nabla_{\mathbb S}L_y
=
\Pi_z\nabla_z L_y.
\end{equation}

Under the joint dimension--concentration regime,
\[
\kappa_c=\nu\lambda_c,
\]
Theorem~1 gives
\[
q_c
=
q_c(0)
+
g_c\rho_c
+
r_c(\rho_c),
\qquad
g_c=\frac{A_c}{\tau},
\]
with
\[
\partial_{\rho_c}r_c
=
O(\nu^{-1}).
\]

Hence
\[
\partial_{\rho_c}q_c
=
g_c
+
O(\nu^{-1}),
\]
and, for a fixed number of classes,
\[
P_c
=
P_c^{\mathrm{lin}}
+
O(\nu^{-1}),
\]
where
\[
P^{\mathrm{lin}}
=
\operatorname{softmax}(\ell),
\qquad
\ell_c
=
a_c
+
g_c\mu_c^\top z.
\]

Substituting into the tangent gradient yields
\begin{equation}
\nabla_{\mathbb S}L_y
=
\Pi_z
\sum_c
\left(
P_c^{\mathrm{lin}}
-
\mathbf 1[c=y]
\right)
g_c\mu_c
+
O(\nu^{-1}).
\label{eq:leading-sphere-gradient}
\end{equation}

Equation~\eqref{eq:leading-sphere-gradient} exposes two distinct roles of
the angular gain.
First, $g_c$ directly rescales the class direction $\mu_c$ in the gradient.
Second, changing $g_c$ modifies the Softmax probabilities and therefore
changes the weights assigned to all competing directions.

Thus even if the gain of the target class is unchanged, modifying the gains
of competing classes changes the local training signal.

\subsection{Ambient and Tangent Decomposition}
\label{sec:ambient-tangent-decomposition}

For $\kappa_c>0$, the exact loss gradient can also be written in the
prototype coordinates.
Let
\[
q_{\rho,c}
=
\partial_{\rho_c}q_c,
\]
and define
\[
C_c
=
\left(
P_c-\mathbf 1[c=y]
\right)
q_{\rho,c}.
\]
Then
\begin{equation}
\nabla_z L_y
=
\sum_c C_c\mu_c
+
z\sum_c\frac{t}{\kappa_c}C_c.
\label{eq:gradient-coordinate-decomposition}
\end{equation}

The second term is radial.
After projection onto the tangent space,
\[
\Pi_z z=0,
\]
and therefore
\begin{equation}
\nabla_{\mathbb S}L_y
=
\Pi_z
\sum_c C_c\mu_c.
\end{equation}

This decomposition explains why the angular derivative is the relevant
quantity for normalized representations even though the exact ambient
gradient contains an additional radial component.

\subsection{Conditional Margin Consequence}
\label{sec:conditional-margin}

We next connect the deterministic boundary geometry to a population-level
classification statement under a matched binary vMF model.

Assume
\[
Z\mid Y=j
\sim
\mathrm{vMF}_p(\mu_j,\nu\lambda_j),
\qquad
j\in\{h,t\},
\]
with fixed $0<t\le T<\infty$, fixed $\lambda_h,\lambda_t>0$, and fixed
class offsets $b_h,b_t$.
Embed $\mu_h$ and $\mu_t$ in a fixed finite-dimensional subspace as
$p\to\infty$.

Write
\[
r_j
=
r_0(\lambda_j),
\qquad
c
=
\mu_h^\top\mu_t,
\]
and let
\[
b
=
b_h-b_t,
\qquad
\delta
=
\frac{b}{t}.
\]

Define
\[
w
=
r_h\mu_h-r_t\mu_t.
\]
The leading binary margin is
\begin{equation}
D_0(z)
=
b+t\,w^\top z.
\label{eq:leading-binary-margin}
\end{equation}

The typical signed margins for the two classes are
\begin{equation}
R_h
=
\delta
+
r_h(r_h-r_tc),
\end{equation}
and
\begin{equation}
R_t
=
r_t(r_t-r_hc)
-
\delta.
\end{equation}

The additional factors $r_h$ and $r_t$ distinguish the behavior of a
random class sample from that of the class prototype itself.

\begin{theorem}[Conditional error transition]
\label{thm:conditional-error-transition}
Under the matched binary model above and a nondegenerate boundary
$w\ne0$, let
$D_\nu$ denote the exact score difference that predicts class $h$ when
$D_\nu\ge 0$.
Then
\[
\Pr_t(D_\nu\ge 0)
\longrightarrow
\begin{cases}
0, & R_t>0,\\[3pt]
\frac12, & R_t=0,\\[3pt]
1, & R_t<0.
\end{cases}
\]
An analogous statement holds for the head-class error.
\end{theorem}

The transition follows because the score approximation error is
$O(\nu^{-1})$, whereas the critical fluctuations of a fixed projection are
$O(\nu^{-1/2})$.

\subsection{Projection Limit}
\label{sec:projection-limit}

For a fixed bounded vector $v$ and
\[
Z\sim\mathrm{vMF}_p(\mu,\nu\lambda),
\]
the centered projection satisfies
\begin{equation}
\sqrt{\nu}
\left[
v^\top Z
-
r_0(\lambda)v^\top\mu
\right]
\Rightarrow
\mathcal N
\left(
0,
\sigma_\lambda^2(v)
\right),
\label{eq:projection-clt}
\end{equation}
where
\begin{equation}
\sigma_\lambda^2(v)
=
r_0'(\lambda)
(v^\top\mu)^2
+
\frac{1-r_0(\lambda)^2}{2}
\left(
\|v\|^2
-
(v^\top\mu)^2
\right).
\end{equation}

\begin{proof}[Proof of the projection limit]
Write $r=r_0(\lambda)$ and decompose
$Z=U\mu+\sqrt{1-U^2}V$, where $V$ is uniform on the unit sphere in
$\mu^\perp$ and independent of $U$. The density of $U$ is proportional to
\[
e^{\nu\lambda u}(1-u^2)^{\nu-1/2}
 =e^{\nu\phi(u)}(1-u^2)^{-1/2},\qquad
\phi(u)=\lambda u+\log(1-u^2),\quad -1<u<1.
\]
The function $\phi$ is strictly concave, has its unique maximizer at $r$,
and satisfies
\[
\phi''(r)=-\frac{2(1+r^2)}{(1-r^2)^2},\qquad
-\frac1{\phi''(r)}=r_0'(\lambda).
\]
For $u=r+x/\sqrt\nu$, Taylor's formula gives
$\nu[\phi(u)-\phi(r)]=\phi''(r)x^2/2+O(|x|^3/\sqrt\nu)$
on bounded $x$-intervals. Strict concavity bounds the integrand by a
Gaussian in a fixed neighborhood of $r$; outside that neighborhood,
$\phi(u)$ has a strictly smaller maximum. The integrable endpoint factor
$(1-u^2)^{-1/2}$ therefore contributes exponentially negligible tails.
Normalizing these integrals proves
$\sqrt\nu(U-r)\Rightarrow N(0,r_0'(\lambda))$.

Put $a=v^\top\mu$ and $v_\perp=v-a\mu$.
Represent $V=G/\|G\|$ for a standard Gaussian vector $G$ of dimension
$p-1=2\nu+1$, chosen independently of $U$. Since
$\|G\|/\sqrt{2\nu+1}\to1$ in probability,
$\sqrt\nu\,v_\perp^\top V\Rightarrow
N(0,\|v_\perp\|^2/2)$, independently of the longitudinal limit.
Consequently,
\[
\sqrt\nu(v^\top Z-ra)
=a\sqrt\nu(U-r)+\sqrt{1-U^2}\sqrt\nu\,v_\perp^\top V
\]
converges to the sum of those independent normal variables, with precisely
the variance in Eq.~\eqref{eq:projection-clt}.
\end{proof}

\begin{proof}[Proof of Theorem~\ref{thm:conditional-error-transition}]
By Theorem~\ref{thm:joint-score-expansion},
$\sup_z|D_\nu(z)-D_0(z)|\le C/\nu$.
Under class $t$, $D_0(Z)\to-tR_t$ in probability. This proves the limits
zero and one when $R_t$ is positive and negative, respectively.
If $R_t=0$, Eq.~\eqref{eq:projection-clt} gives
$\sqrt\nu D_\nu/t\Rightarrow N(0,\sigma_{\lambda_t}^2(w))$.
For $w\ne0$ this variance is positive, so continuity of the normal
cdf at zero gives the limit $1/2$. For class $h$,
$D_0(Z)\to tR_h$, and the same argument applied to the event $D_\nu<0$
proves the corresponding head-class statement.
\end{proof}

\begin{figure}[H]
  \centering
  \includegraphics[width=.82\linewidth]{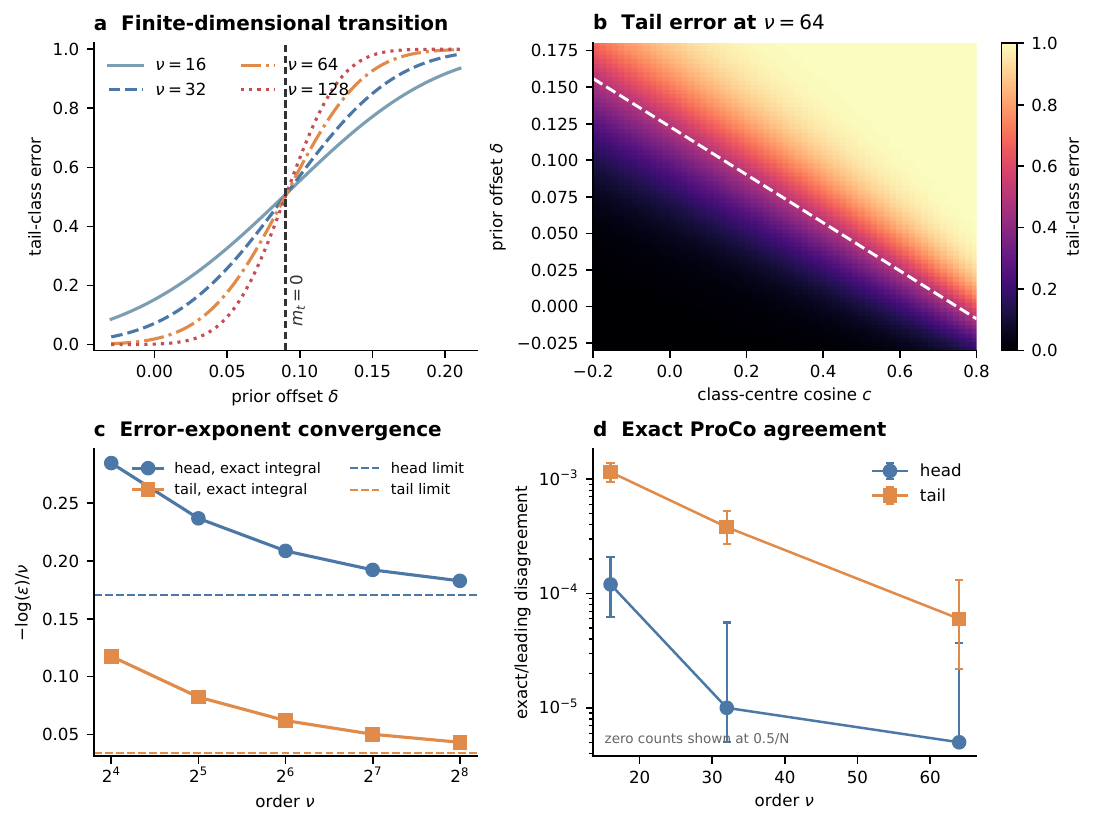}
  \caption{Finite-dimensional approach to the tail-error transition.
  (a,b) The transition narrows around $R_t=0$. (c) Exact one-dimensional
  errors approach the large-deviation exponents. (d) Exact ProCo and the
  leading rule agree.  Zero Monte Carlo counts are displayed at half a count
  and are not treated as zero probability.}
  \label{fig:synthetic-transition}
\end{figure}

\subsection{Large-Deviation Exponent}
\label{sec:large-deviation-exponent}

For completeness, let
\[
W=\|w\|,
\qquad
n=\frac{w}{W},
\qquad
\beta_j=\mu_j^\top n,
\qquad
\alpha_j=\sqrt{1-\beta_j^2},
\]
and
\[
\xi_0
=
-\frac{\delta}{W}.
\]

Define
\[
J(s)
=
\sqrt{1+s^2}
-
\log\left(
1+\sqrt{1+s^2}
\right),
\]
and
\begin{equation}
I_j(\xi)
=
J(\lambda_j)
-
J\left(
\lambda_j\alpha_j\sqrt{1-\xi^2}
\right)
-
\lambda_j\beta_j\xi
-
\log(1-\xi^2).
\label{eq:rate-function}
\end{equation}

If both typical margins are positive, the class-conditional error
probabilities satisfy
\begin{equation}
\lim_{\nu\to\infty}
-\frac{1}{\nu}\log\epsilon_j
=
I_j(\xi_0).
\label{eq:error-exponent}
\end{equation}

When
\[
\lambda_h\ge\lambda_t,
\qquad
b_h\ge b_t,
\qquad
R_t>0,
\]
the exponents obey
\[
I_h(\xi_0)
\ge
I_t(\xi_0),
\]
with equality only in the symmetric case.

\begin{proof}[Derivation of the error exponent and its ordering]
Choose a unit vector $e\perp n$ such that
$\mu_j=\beta_j n+\alpha_j e$ (any such $e$ is valid if $\alpha_j=0$).
The first two coordinates $(\xi,u)=(n^\top Z,e^\top Z)$ have density
proportional to
\[
\exp\{\nu\lambda_j(\beta_j\xi+\alpha_j u)\}
(1-\xi^2-u^2)^{\nu-1},\qquad \xi^2+u^2<1.
\]
Its limiting log-density, before normalization, is
$F_j(\xi,u)=\lambda_j(\beta_j\xi+\alpha_j u)
+\log(1-\xi^2-u^2)$.
Laplace's principle applies on the disk: on compact interior subsets it
is the usual smooth Laplace estimate; a boundary layer with
$1-\xi^2-u^2<\eta$ has integral bounded by a polynomial factor times
$e^{\nu\lambda_j}\eta^\nu$, and can be discarded as $\eta\downarrow0$.
For $a=1-\xi^2$ and $s=\lambda_j\alpha_j\sqrt a$, maximizing over
$u=\sqrt a\,v$ gives
\[
\sup_u F_j(\xi,u)
=\lambda_j\beta_j\xi+\log a+J(s)+\log2-1,
\]
since $\sup_{|v|<1}\{sv+\log(1-v^2)\}=J(s)+\log2-1$.
The unconstrained maximum is $J(\lambda_j)+\log2-1$.
Their difference is exactly $I_j(\xi)$ in Eq.~\eqref{eq:rate-function}.

The function $F_j$ is strictly concave, so the contracted rate is strictly
convex on $(-1,1)$, with unique minimum zero at $\xi=r_j\beta_j$.
Positive typical margins imply
$r_t\beta_t<\xi_0<r_h\beta_h$, in particular $|\xi_0|<1$.
Thus the rate infimum over the head error interval $\xi\le\xi_0$,
or the tail error interval $\xi\ge\xi_0$, equals $I_j(\xi_0)$.
The open and closed error intervals have the same infimum by continuity,
which proves the exponent for the leading rule.
For the exact rule, the uniform score bound gives, for example,
\[
\{D_0<-C/\nu\}\subseteq\{D_\nu<0\}
\subseteq\{D_0<C/\nu\}.
\]
Sandwiching these events by fixed threshold perturbations $\xi_0\pm\epsilon$,
applying the preceding interval bounds, and then letting $\epsilon\downarrow0$
proves Eq.~\eqref{eq:error-exponent} for the exact score as well.

For the exponent ordering, first take $\delta=0$ and define
$K(\lambda,a)=J(\lambda)-J(\lambda\sqrt{1-a^2})$, $0\le a\le1$.
Then $I_h(0)=K(\lambda_h,\beta_h)$ and
$I_t(0)=K(\lambda_t,|\beta_t|)$.
Under the stated assumptions $\beta_h>0>\beta_t$, and
\[
\beta_h-|\beta_t|=\frac{(r_h-r_t)(1+c)}{W}\ge0.
\]
Writing $S(x)=\sqrt{1+x^2}$, direct differentiation gives
\[
\partial_\lambda K
=\frac{S(\lambda)-S(\lambda\sqrt{1-a^2})}{\lambda}\ge0,
\qquad
\partial_a K
=\frac{\lambda^2a}{1+S(\lambda\sqrt{1-a^2})}\ge0.
\]
Hence $I_h(0)\ge I_t(0)$, strictly if $\lambda_h>\lambda_t$.
Increasing $\delta$ from zero moves $\xi_0=-\delta/W$ to the left,
away from the head typical point and toward the tail typical point.
As long as $R_t>0$, strict convexity increases the head rate and decreases
the tail rate. Equality therefore requires both
$\lambda_h=\lambda_t$ and $b_h=b_t$, the symmetric case.
\end{proof}

The synthetic experiments in the supplementary results verify both the
finite-dimensional transition and the exponent convergence.
\begin{figure}[H]
  \centering
  \includegraphics[width=0.66\linewidth]{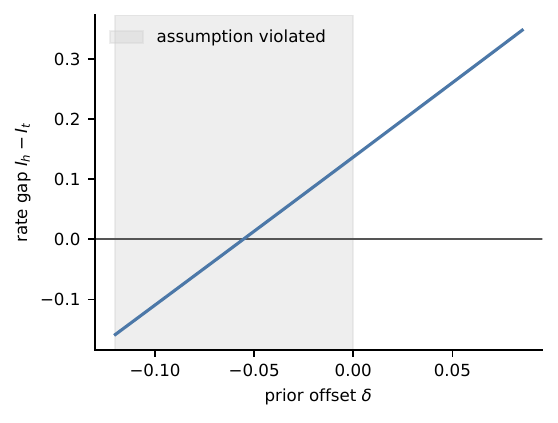}
  \caption{Reversing the prior ordering can reverse the head--tail exponent
  order; the prior condition in the theorem is substantive.}
  \label{fig:prior-control-app}
\end{figure}

\section{Cosine Equivalence and Gain Realizations}
\label{app:cosine-equivalence}
\label{sec:cosine-equivalence}

This section studies complete gain equalization at fixed representations
and class statistics.
We first separate the three score realizations exactly, then characterize
when Pure Angular equalization agrees with a shared-scale cosine prototype
classifier, and finally quantify the remaining loss and gradient
differences.

\subsection{Exact Separation of the Three Gain Realizations}
\label{sec:realization-separation}

Let the original score be
\[
q_c^0(\rho)
=
q_c(\rho;\tau_0),
\]
with original gain
\[
g_c
=
\frac{A_c}{\tau_0}.
\]

Define the class-average gain
\[
\bar g
=
\frac1C\sum_c g_c,
\]
and the target gain
\begin{equation}
g_c^\star
=
\bar g
+
\beta(g_c-\bar g).
\label{eq:target-gain}
\end{equation}

The original ProCo score is recovered at
\[
\beta=1,
\]
whereas
\[
\beta=0
\]
gives complete equalization of the leading angular coefficient.

For positive target gains, define
\[
\tau_c^\star
=
\frac{A_c}{g_c^\star}.
\]

The three controlled realizations are
\begin{equation}
q_c^{\mathrm{temp}}(\rho)
=
q_c(\rho;\tau_c^\star),
\end{equation}
\begin{equation}
q_c^{\mathrm{preserve}}(\rho)
=
q_c(\rho;\tau_c^\star)
-
q_c(0;\tau_c^\star)
+
q_c(0;\tau_0),
\end{equation}
and
\begin{equation}
q_c^{\mathrm{angular}}(\rho)
=
q_c^0(\rho)
+
(g_c^\star-g_c)\rho.
\end{equation}

Let
\[
h_c^0
=
q_c(0;\tau_0),
\qquad
h_c^\star
=
q_c(0;\tau_c^\star),
\]
and define the corresponding remainders
\[
r_c^0(\rho)
=
q_c(\rho;\tau_0)
-
h_c^0
-
g_c\rho,
\]
\[
r_c^\star(\rho)
=
q_c(\rho;\tau_c^\star)
-
h_c^\star
-
g_c^\star\rho.
\]

Then the three realizations decompose exactly as
\[
q_c^{\mathrm{temp}}
=
h_c^\star
+
g_c^\star\rho_c
+
r_c^\star,
\]
\[
q_c^{\mathrm{preserve}}
=
h_c^0
+
g_c^\star\rho_c
+
r_c^\star,
\]
and
\begin{equation}
q_c^{\mathrm{angular}}
=
h_c^0
+
g_c^\star\rho_c
+
r_c^0.
\label{eq:three-realizations}
\end{equation}

Consequently,
\begin{equation}
q_c^{\mathrm{temp}}
-
q_c^{\mathrm{preserve}}
=
h_c^\star-h_c^0,
\label{eq:temp-preserve-diff}
\end{equation}
and
\begin{equation}
q_c^{\mathrm{preserve}}
-
q_c^{\mathrm{angular}}
=
r_c^\star-r_c^0.
\label{eq:preserve-angular-diff}
\end{equation}

Thus direct temperature adjustment changes both the intrinsic intercept and
the finite-dimensional response.
Intercept preservation restores the original cosine-zero score but retains
the new remainder.
Pure Angular preserves both the original intercept and the original
remainder while changing only the specified leading linear coefficient.

\subsection{Intercept Shifts and Loss Gradients}
\label{sec:intercept-gradient-effect}

Direct temperature adjustment and intercept preservation differ only by a
class-dependent constant, so their classwise angular derivatives are
identical:
\[
d_c
=
\partial_{\rho_c}
q_c^{\mathrm{temp}}
=
\partial_{\rho_c}
q_c^{\mathrm{preserve}}.
\]

Nevertheless, their loss gradients need not agree.
With identical external class offsets,
\begin{equation}
\nabla_{\mathbb S}L_y^{\mathrm{temp}}
-
\nabla_{\mathbb S}L_y^{\mathrm{preserve}}
=
\Pi_z
\sum_c
\left(
P_c^{\mathrm{temp}}
-
P_c^{\mathrm{preserve}}
\right)
d_c\mu_c.
\label{eq:intercept-gradient-difference}
\end{equation}

The angular derivatives are unchanged, but the class-dependent intercept
shift changes the Softmax probabilities and therefore changes the weights
assigned to all class directions.

This identity explains why restoring the intercept can substantially alter
the loss even when the underlying score derivatives remain unchanged.

\subsection{Complete Equalization and Cosine Prototypes}
\label{sec:cosine-prediction-equivalence}

At complete equalization,
\[
\beta=0,
\qquad
g_c^\star=\bar g.
\]

Define
\begin{equation}
\delta_c(z)
=
q_c^0(\rho_c)
-
g_c\rho_c
=
h_c^0+r_c^0(\rho_c).
\label{eq:delta-definition}
\end{equation}

Equation~\eqref{eq:three-realizations} becomes
\begin{equation}
q_c^{\mathrm{angular}}(z)
=
\bar g\,\rho_c
+
\delta_c(z).
\label{eq:Pure Angular-cosine-decomposition}
\end{equation}

The leading term is therefore exactly a shared-scale cosine prototype
score.

Let
\[
\rho_{(1)}
\ge
\rho_{(2)}
\]
denote the largest and second-largest prototype cosines, and define
\begin{equation}
m_{\cos}(z)
=
\rho_{(1)}-\rho_{(2)}.
\end{equation}
Also define the residual range
\begin{equation}
D(z)
=
\max_c\delta_c(z)
-
\min_c\delta_c(z).
\end{equation}

\begin{proposition}[Finite-dimensional prediction stability]
\label{prop:cosine-prediction-stability}
Assume $\bar g>0$ and equal class offsets.
If
\begin{equation}
\bar g\,m_{\cos}(z)
>
D(z),
\label{eq:cosine-margin-certificate}
\end{equation}
then the shared-scale cosine prototype classifier and complete Pure Angular
equalization have the same unique winning class.
\end{proposition}

\begin{proof}
Let $j$ maximize $\rho_c$.
For every $c\neq j$,
\[
q_j^{\mathrm{angular}}
-
q_c^{\mathrm{angular}}
=
\bar g(\rho_j-\rho_c)
+
\delta_j-\delta_c.
\]
Since
\[
\rho_j-\rho_c
\ge
m_{\cos}(z),
\]
and
\[
\delta_j-\delta_c
\ge
-D(z),
\]
we obtain
\[
q_j^{\mathrm{angular}}
-
q_c^{\mathrm{angular}}
\ge
\bar g\,m_{\cos}(z)-D(z)
>
0.
\]
Thus $j$ is also the unique Pure Angular winner.
\end{proof}

With the same deterministic tie rule for both classifiers,
Proposition~\ref{prop:cosine-prediction-stability} immediately implies
\begin{equation}
\left|
R_{\mathrm{angular}}
-
R_{\cos}
\right|
\le
\Pr
\left(
\widehat y_{\mathrm{angular}}
\ne
\widehat y_{\cos}
\right),
\end{equation}
and
\begin{equation}
\Pr
\left(
\widehat y_{\mathrm{angular}}
\ne
\widehat y_{\cos}
\right)
\le
\Pr
\left(
m_{\cos}(z)
\le
\frac{D(z)}{\bar g}
\right).
\label{eq:prediction-disagreement-bound}
\end{equation}

Under the conditions of Theorem~1,
\[
\sup_{c,z}
|\delta_c(z)|
=
O(\nu^{-1}).
\]
If
\[
\bar g
\ge
g_{\min}>0,
\]
then disagreement requires
\[
m_{\cos}(z)
=
O(\nu^{-1}).
\]

Thus complete Pure Angular equalization approaches the shared-scale cosine
prototype rule outside a shrinking low-margin region.

For nonuniform external offsets, the appropriate comparator is
\[
b_c+\bar g\rho_c,
\]
and the same argument applies to its top-two logit gap.

\subsection{Finite-Dimensional Loss and Gradient Stability}
\label{sec:cosine-gradient-stability}

We now compare the complete Pure Angular logits
\[
s_c^{\mathrm{angular}}
=
b_c+\bar g\rho_c+\delta_c(\rho_c)
\]
with the shared-scale cosine logits
\[
s_c^{\cos}
=
b_c+\bar g\rho_c.
\]

\begin{proposition}[Fixed-state loss and gradient stability]
\label{prop:cosine-gradient-stability}
Suppose that at a query $z$,
\[
\max_c|\delta_c|
\le
\epsilon,
\]
and
\[
\max_c
|\partial_{\rho_c}\delta_c|
\le
\eta.
\]
Then
\begin{equation}
\left|
L_y^{\mathrm{angular}}
-
L_y^{\cos}
\right|
\le
2\epsilon,
\label{eq:cosine-loss-bound}
\end{equation}
and
\begin{equation}
\left\|
\nabla_{\mathbb S}L_y^{\mathrm{angular}}
-
\nabla_{\mathbb S}L_y^{\cos}
\right\|
\le
2\bar g\,\epsilon
+
2\eta.
\label{eq:cosine-gradient-bound}
\end{equation}
\end{proposition}

\begin{proof}
The logit perturbation has
\[
\|\delta\|_\infty
\le
\epsilon.
\]
Log-sum-exp is $1$-Lipschitz in the $\ell_\infty$ norm, while the target
logit changes by at most $\epsilon$, giving
Eq.~\eqref{eq:cosine-loss-bound}.

For the gradient, let
\[
P^{\mathrm{angular}}
=
\operatorname{softmax}(s^{\mathrm{angular}}),
\qquad
P^{\cos}
=
\operatorname{softmax}(s^{\cos}).
\]
The Softmax Jacobian satisfies
\[
\|Jv\|_1
\le
2\|v\|_\infty,
\]
hence
\[
\|P^{\mathrm{angular}}-P^{\cos}\|_1
\le
2\epsilon.
\]

Subtracting the two tangent gradients gives
\[
\Pi_z
\left[
\bar g
\sum_c
\left(
P_c^{\mathrm{angular}}
-
P_c^{\cos}
\right)\mu_c
+
\sum_c
\left(
P_c^{\mathrm{angular}}
-
\mathbf 1[c=y]
\right)
\partial_{\rho_c}\delta_c\,
\mu_c
\right].
\]
Using
\[
\|\Pi_z\|\le 1,
\qquad
\|\mu_c\|=1,
\]
and
\[
\|P^{\mathrm{angular}}-e_y\|_1
\le 2,
\]
yields Eq.~\eqref{eq:cosine-gradient-bound}.
\end{proof}

By Theorem~1,
\[
\epsilon
=
O(\nu^{-1}),
\qquad
\eta
=
O(\nu^{-1}),
\]
so both the loss difference and local gradient difference are
$O(\nu^{-1})$ in the joint high-dimensional regime.

\subsection{Frozen Score-Realization Comparison}
\label{sec:frozen-realization-comparison}

The frozen experiments compare all score realizations on identical
representations and fitted class statistics.
No representation is changed within a row.

\begin{table}[t]
\centering\small
\setlength{\tabcolsep}{4pt}
\caption{\textbf{Frozen accuracy across score realizations.} Accuracy (\%) at $\beta=0$ and zero prior offset. $\Delta$ is Pure Angular minus original vMF in percentage points; darker green indicates a larger increase.}
\label{tab:frozen-vmf-realizations}
\begin{tabular}{@{}lcccccr@{}}
\toprule
Dataset / IF & Feature & Original & $q^{\rm temp}$ & $q^{\rm preserve}$ & $q^{\rm angular}$ & $\Delta$\\
\midrule
CIFAR-10 / 10 & Encoder & 85.21 & 85.92 & 85.84 & 85.90 & \cellcolor{gainshade!14.3}$+0.69$\\
 & Projection & 87.85 & 87.83 & 87.83 & 87.83 & $-0.02$\\
CIFAR-10 / 50 & Encoder & 81.47 & 80.71 & 80.72 & 80.72 & $-0.75$\\
 & Projection & 80.72 & 80.58 & 80.58 & 80.58 & $-0.14$\\
CIFAR-10 / 100 & Encoder & 78.76 & 78.00 & 78.12 & 78.02 & $-0.74$\\
 & Projection & 77.81 & 77.58 & 77.59 & 77.58 & $-0.23$\\
CIFAR-100 / 10 & Encoder & 48.87 & 52.50 & 52.48 & 52.50 & \cellcolor{gainshade!32.9}$+3.63$\\
 & Projection & 54.38 & 55.09 & 55.08 & 55.15 & \cellcolor{gainshade!14.8}$+0.77$\\
CIFAR-100 / 50 & Encoder & 39.96 & 44.35 & 44.31 & 44.35 & \cellcolor{gainshade!37.6}$+4.39$\\
 & Projection & 42.90 & 43.65 & 43.61 & 43.63 & \cellcolor{gainshade!14.6}$+0.73$\\
CIFAR-100 / 100 & Encoder & 36.77 & 41.59 & 41.49 & 41.59 & \cellcolor{gainshade!40.3}$+4.82$\\
 & Projection & 39.80 & 40.89 & 40.86 & 40.90 & \cellcolor{gainshade!16.9}$+1.10$\\
ImageNet-LT & Encoder & 41.39 & 45.94 & 45.98 & 45.97 & \cellcolor{gainshade!38.8}$+4.58$\\
 & Projection & 46.14 & 46.40 & 46.39 & 46.39 & \cellcolor{gainshade!11.6}$+0.25$\\
iNaturalist 2018 & Encoder & 57.61 & 64.78 & 64.76 & 64.76 & \cellcolor{gainshade!55.0}$+7.15$\\
 & Projection & 63.07 & 65.42 & 65.40 & 65.42 & \cellcolor{gainshade!24.8}$+2.35$\\
\bottomrule
\end{tabular}
\end{table}

The Direct Temperature, Intercept-Preserving, and Pure Angular rules produce
similar decisions in many normally trained frozen states, whereas the
same-state failure example in the main text shows that an aggressive
temperature realization can introduce a large intercept shift.

\subsection{Where Do Cosine and Pure Angular Disagree?}
\label{sec:cosine-diagnostics}

We next evaluate the margin certificate from
Eq.~\eqref{eq:cosine-margin-certificate} on every held-out query.

\begin{table}[t]
\centering\small
\setlength{\tabcolsep}{4pt}
\caption{\textbf{Prediction disagreements and margin certificates.} ``Low-margin'' reports the number of disagreements falling in the lowest cosine-gap decile. Certified queries satisfy $\bar g m_{\cos}>D$; all certified queries agree.}
\label{tab:cosine-diagnostics}
\begin{tabular}{@{}lcccr@{}}
\toprule
Dataset / IF & Feature & Disagree & Low-margin & Certified / Total\\
\midrule
CIFAR-10 / 10 & Encoder & 26 & 26 & 9,763 / 10,000\\
 & Projection & 1 & 1 & 9,982 / 10,000\\
CIFAR-10 / 50 & Encoder & 54 & 54 & 9,717 / 10,000\\
 & Projection & 5 & 5 & 9,969 / 10,000\\
CIFAR-10 / 100 & Encoder & 49 & 49 & 9,655 / 10,000\\
 & Projection & 7 & 7 & 9,960 / 10,000\\
CIFAR-100 / 10 & Encoder & 124 & 124 & 6,348 / 10,000\\
 & Projection & 68 & 68 & 8,857 / 10,000\\
CIFAR-100 / 50 & Encoder & 157 & 155 & 5,874 / 10,000\\
 & Projection & 100 & 100 & 8,501 / 10,000\\
CIFAR-100 / 100 & Encoder & 148 & 144 & 6,061 / 10,000\\
 & Projection & 135 & 135 & 8,272 / 10,000\\
ImageNet-LT & Encoder & 151 & 151 & 47,832 / 50,000\\
 & Projection & 60 & 60 & 48,427 / 50,000\\
iNaturalist 2018 & Encoder & 27 & 27 & 23,635 / 24,426\\
 & Projection & 38 & 38 & 23,594 / 24,426\\
\bottomrule
\end{tabular}
\end{table}

Prediction disagreement is rare and strongly localized to low cosine
margins.
Every query satisfying the sufficient certificate has identical cosine and
Pure Angular predictions.

The result holds both when equalization improves accuracy and when it
decreases accuracy relative to the original vMF score.
The cosine connection therefore describes the geometry induced by complete
gain equalization rather than the direction of its accuracy change.

\subsection{Local Gradient Comparison}
\label{sec:cosine-gradient-experiments}

Finally, we compare cross-entropy query gradients for the cosine and Pure
Angular scores on identical frozen states.

\begin{table}[t]
\centering\small
\setlength{\tabcolsep}{5pt}
\caption{\textbf{Loss and gradient differences between Pure Angular and cosine.} Complete equalization with zero prior offset. Gradient error is $100\|G_{\rm angular}-G_{\cos}\|_F/\|G_{\rm angular}\|_F$, with query tangent gradients as rows of $G$. The last column reports their median cosine similarity.}
\label{tab:cosine-gradients}
\begin{tabular}{@{}lcccr@{}}
\toprule
Dataset / IF & Feature & Mean $|\Delta L|$ & Grad. error (\%) & Median grad. cosine\\
\midrule
CIFAR-100 / 10 & Encoder & 0.01141 & 1.555 & 0.99999712\\
 & Projection & 0.01751 & 1.624 & 0.99996033\\
CIFAR-100 / 50 & Encoder & 0.00917 & 1.524 & 0.99999679\\
 & Projection & 0.02206 & 1.592 & 0.99996161\\
CIFAR-100 / 100 & Encoder & 0.00832 & 1.448 & 0.99999724\\
 & Projection & 0.02634 & 1.595 & 0.99995594\\
ImageNet-LT & Encoder & 0.00217 & 0.146 & 0.99999978\\
 & Projection & 0.00232 & 0.122 & 0.99999967\\
iNaturalist 2018 & Encoder & 0.00362 & 0.132 & 0.99999975\\
 & Projection & 0.00456 & 0.143 & 0.99999973\\
\bottomrule
\end{tabular}
\end{table}

On CIFAR-100, the stacked-query relative gradient difference is approximately
$1.45$--$1.62\%$.
On ImageNet-LT and iNaturalist it is approximately
$0.12$--$0.15\%$, with median gradient cosine similarity above $0.99995$.

Thus the high prediction agreement after complete equalization is
accompanied by closely aligned local learning signals.
\section{Additional Proofs and Finite-Sample Effects}
\label{sec:additional-proofs}

This appendix collects three technical results used by the main analysis:
the complete multiclass propagation of the joint score expansion, a uniform
remainder bound for the adjacent Bessel ratio, and finite-sample identities
that separate population geometry from statistical support.

\subsection{Complete Multiclass Propagation}
\label{sec:multiclass-expansion}

Let
\[
\epsilon
=
\nu^{-1},
\]
and suppose the class score expansion is
\[
s_j
=
\ell_j^0
+
\epsilon B_{1j}
+
\epsilon^2B_{2j}
+
O(\epsilon^3),
\]
where
\[
\ell_j^0
=
b_j+B_{0j}.
\]

Define
\[
P_j^0
=
\frac{e^{\ell_j^0}}
{\sum_k e^{\ell_k^0}},
\]
\[
m_r
=
\sum_j P_j^0B_{rj},
\]
\[
\Delta_{rj}
=
B_{rj}-m_r,
\]
and
\[
V_1
=
\sum_j
P_j^0\Delta_{1j}^2.
\]

\begin{theorem}[Complete multiclass propagation]
\label{thm:multiclass-propagation}
For a fixed number of classes and uniformly over the compact parameter
family of Theorem~\ref{thm:joint-score-expansion},
\begin{equation}
P_j
=
P_j^0
+
\epsilon
P_j^0\Delta_{1j}
+
\epsilon^2
P_j^0
\left[
\Delta_{2j}
+
\frac{\Delta_{1j}^2-V_1}{2}
\right]
+
O(\epsilon^3),
\label{eq:multiclass-prob-expansion}
\end{equation}
and
\begin{equation}
L_y
=
L_y^0
+
\epsilon(m_1-B_{1y})
+
\epsilon^2
\left(
m_2-B_{2y}
+
\frac{V_1}{2}
\right)
+
O(\epsilon^3),
\label{eq:multiclass-loss-expansion}
\end{equation}
where
\[
L_y^0
=
-\ell_y^0
+
\log\sum_j e^{\ell_j^0}.
\]
\end{theorem}

\begin{proof}
Write
\[
u_j
=
\epsilon B_{1j}
+
\epsilon^2B_{2j}
+
O(\epsilon^3).
\]
Then
\[
e^{u_j}
=
1
+
\epsilon B_{1j}
+
\epsilon^2
\left(
B_{2j}
+
\frac{B_{1j}^2}{2}
\right)
+
O(\epsilon^3).
\]

Normalizing
\[
P_j^0e^{u_j}
\]
by the corresponding sum gives
\[
P_j
=
P_j^0
\left[
1
+
\epsilon\Delta_{1j}
+
\epsilon^2
\left(
\Delta_{2j}
+
\frac{\Delta_{1j}^2-V_1}{2}
\right)
\right]
+
O(\epsilon^3),
\]
which proves
Eq.~\eqref{eq:multiclass-prob-expansion}.

For the loss,
\[
\log\sum_j
e^{\ell_j^0+u_j}
=
\log\sum_j e^{\ell_j^0}
+
\epsilon m_1
+
\epsilon^2
\left(
m_2+\frac{V_1}{2}
\right)
+
O(\epsilon^3).
\]
Subtracting the true-class score proves
Eq.~\eqref{eq:multiclass-loss-expansion}.
\end{proof}

The term
\[
\frac{V_1}{2}
=
\frac12
\operatorname{Var}_{P^0}(B_1)
\]
comes from the curvature of log-sum-exp.
It couples the classwise first-order score corrections and is required to
recover the full second-order multiclass loss.

Equivalently, for
\[
F(s)
=
\log\sum_j e^{s_j},
\]
the Hessian at $\ell^0$ is
\[
H
=
\operatorname{Diag}(P^0)
-
P^0(P^0)^\top,
\]
and
\[
B_1^\top H B_1
=
\operatorname{Var}_{P^0}(B_1)
=
V_1.
\]

The numerical validation of this coupling term is reported with the
supplementary mechanism experiments.

\begin{figure}[H]
  \centering
  \includegraphics[width=0.82\linewidth]{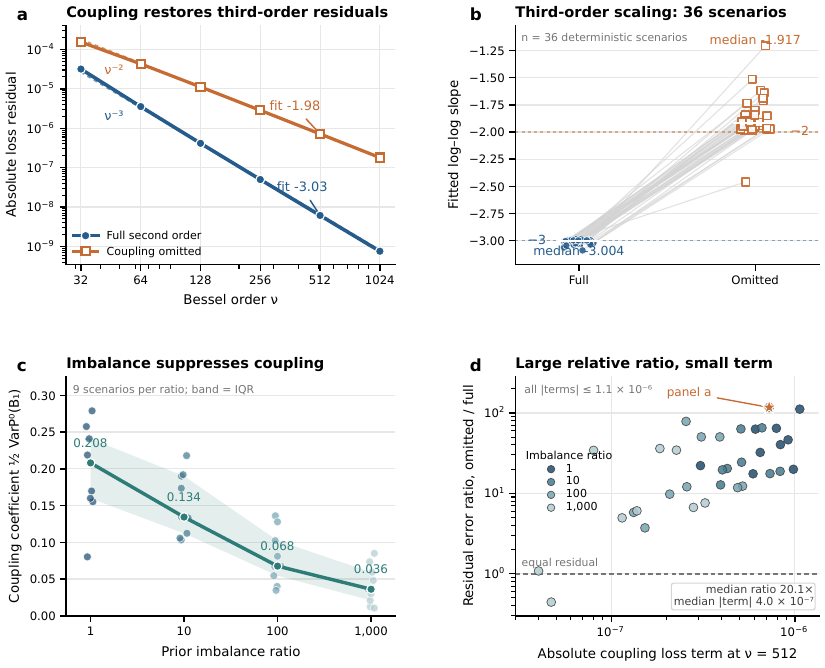}
  \caption{\textbf{Variance coupling in the multiclass expansion.} Retaining the second-order variance term recovers the predicted third-order residual decay.}
  \label{fig:softmax-coupling}
\end{figure}

\subsection{Uniform Remainder for the Bessel Ratio}
\label{sec:bessel-ratio-remainder}

The joint expansion in Appendix~\ref{sec:joint-expansion} uses a uniform
remainder bound for
\[
R_\nu(x)
=
\frac{I_{\nu+1}(x)}{I_\nu(x)}.
\]

\begin{lemma}[Uniform adjacent-ratio remainder]
\label{lem:bessel-ratio-remainder}
For
\[
\nu\ge\frac{10}{9},
\]
define
\begin{equation}
A_{2,\nu}(x)
=
\frac{x}
{\nu+\sqrt{\nu^2+x^2}}
-
\frac{x}
{2(\nu^2+x^2)}
+
\frac{x(4\nu^2-x^2)}
{8(\nu^2+x^2)^{5/2}}.
\label{eq:bessel-ratio-approx}
\end{equation}
Then
\begin{equation}
\left|
R_\nu(x)
-
A_{2,\nu}(x)
\right|
\le
\nu^{-3}
\qquad
\text{for all }x\ge 0.
\label{eq:bessel-ratio-bound}
\end{equation}
\end{lemma}

\begin{proof}
Set
\[
\varepsilon
=
\nu^{-1},
\qquad
x=\nu z,
\qquad
H_\varepsilon(z)
=
R_\nu(\nu z).
\]
The adjacent-ratio Riccati equation becomes
\begin{equation}
\varepsilon H_\varepsilon'
=
1
-
\frac{2+\varepsilon}{z}
H_\varepsilon
-
H_\varepsilon^2.
\label{eq:scaled-riccati}
\end{equation}

Let
\[
a_\varepsilon(z)
=
r_0(z)
+
\varepsilon r_1(z)
+
\varepsilon^2r_2(z).
\]
Substituting $a_\varepsilon$ into
Eq.~\eqref{eq:scaled-riccati} gives the residual
\begin{equation}
E
=
-
\varepsilon^3
\frac{z^4-10z^2+4}
{4(1+z^2)^{7/2}}
-
\varepsilon^4
\frac{z^2(z^2-4)^2}
{64(1+z^2)^5}.
\end{equation}
Elementary bounds imply
\begin{equation}
|E|
\le
\varepsilon^3
\left(
1+\frac{\varepsilon}{4}
\right).
\end{equation}

For
\[
e
=
H_\varepsilon-a_\varepsilon,
\]
we obtain
\begin{equation}
\varepsilon e'
=
E
-
Ke
-
e^2,
\label{eq:error-barrier-ode}
\end{equation}
with
\[
K
=
\frac{2\sqrt{1+z^2}}{z}
+
\frac{\varepsilon}
{z(1+z^2)}
+
2\varepsilon^2r_2(z).
\]

For
\[
\varepsilon\le\frac{9}{10},
\]
the coefficient $K$ remains uniformly positive.
At the barriers
\[
e=\pm\varepsilon^3,
\]
Eq.~\eqref{eq:error-barrier-ode} points inward.
Since
\[
e(0)=0,
\]
a first-contact argument prevents either barrier from being crossed.

Hence
\[
|e(z)|
\le
\varepsilon^3
\]
for all $z\ge 0$, proving
Eq.~\eqref{eq:bessel-ratio-bound}.
\end{proof}

In the joint scaling $x=\nu\lambda$,
Eq.~\eqref{eq:bessel-ratio-approx} reduces exactly to
\[
r_0(\lambda)
+
\frac{r_1(\lambda)}{\nu}
+
\frac{r_2(\lambda)}{\nu^2},
\]
which supplies the ratio expansion used in
Theorem~\ref{thm:joint-score-expansion}.

\subsection{Finite-Sample Resultant Identity}
\label{sec:finite-sample}

Let
\[
X_1,\ldots,X_n
\]
be independent unit vectors satisfying
\[
\mathbb E[X_i]
=
A\mu,
\qquad
\|\mu\|=1.
\]
Define
\[
\overline X_n
=
\frac1n\sum_{i=1}^n X_i,
\qquad
R_n
=
\|\overline X_n\|.
\]

Expanding the squared norm gives
\begin{equation}
\mathbb E[R_n^2]
=
A^2
+
\frac{1-A^2}{n}.
\label{eq:finite-resultant-identity}
\end{equation}

Indeed,
\[
\mathbb E
\left\|
\sum_{i=1}^n X_i
\right\|^2
=
n
+
n(n-1)A^2,
\]
and division by $n^2$ yields
Eq.~\eqref{eq:finite-resultant-identity}.

It follows that
\begin{equation}
\mathbb E
\left[
\frac{nR_n^2-1}{n-1}
\right]
=
A^2.
\label{eq:unbiased-resultant-square}
\end{equation}

Thus the population resultant $A$ and the support size $n$ play distinct
roles:
$A$ describes directional geometry, whereas $n$ controls the finite-sample
bias and uncertainty of the observed resultant.

\subsection{Current-Query Self-Inclusion}
\label{sec:self-inclusion}

Let
\[
\overline X_{-i}
=
\frac{1}{n-1}
\sum_{j\ne i}X_j.
\]
If the current observation $X_i$ is inserted into the class mean before it
is scored, then the leading linear score differs from the leave-one-out
version by
\[
\frac{
X_i^\top\overline X_n
-
X_i^\top\overline X_{-i}
}{\tau}.
\]

Taking expectation gives
\begin{equation}
\mathbb E
\left[
\frac{
X_i^\top\overline X_n
-
X_i^\top\overline X_{-i}
}{\tau}
\right]
=
\frac{1-A^2}{\tau n}.
\label{eq:self-inclusion}
\end{equation}

This identity requires only independent unit-vector sampling and does not
require a single-vMF population model.

The effect is therefore an estimation phenomenon whose magnitude scales as
$1/n$.

\subsection{Cross-Fitted Gain Equalization}
\label{sec:cross-fitted-estimator}

A cross-fitted construction can separate the score of each example from
the statistics used to estimate its class geometry.

For a held-out fold $k$ of class $c$, let
\[
\overline z_c^{(-k)}
=
\frac{1}{n_c^{(-k)}}
\sum_{\substack{i:y_i=c\\i\notin k}}
z_i,
\]
\[
R_c^{(-k)}
=
\left\|
\overline z_c^{(-k)}
\right\|,
\]
and
\[
\widehat\mu_c^{(-k)}
=
\frac{
\overline z_c^{(-k)}
}{
R_c^{(-k)}
}.
\]

Define
\[
U_c^{(-k)}
=
\frac{
n_c^{(-k)}
(R_c^{(-k)})^2
-
1
}{
n_c^{(-k)}-1
},
\]
and
\[
\widehat A_c^{(-k)}
=
\sqrt{
\max
\left\{
U_c^{(-k)},0
\right\}
}.
\]

Given a reference resultant $A_{\mathrm{ref}}>0$, define
\begin{equation}
\tau_c^{(-k)}
=
\tau
\operatorname{clip}
\left(
\frac{
\widehat A_c^{(-k)}
}{
A_{\mathrm{ref}}
},
r_{\min},
r_{\max}
\right).
\label{eq:cross-fit-temperature}
\end{equation}

\begin{proposition}[Conditional leading-gain equalization]
\label{prop:cross-fit-equalization}
Suppose
\[
\widehat\mu_c^{(-k)}
=
\mu_c
+
o_p(1),
\]
\[
\widehat A_c^{(-k)}
=
A_c
+
o_p(1),
\]
and the fitted concentration remains in the compact joint regime.
If clipping in
Eq.~\eqref{eq:cross-fit-temperature}
is inactive and the resulting inverse temperatures remain bounded, then
uniformly over unit $z$,
\begin{equation}
q_c(z;\tau_c^{(-k)})
=
\frac{A_{\mathrm{ref}}}{\tau}
\mu_c^\top z
+
O(\nu^{-1})
+
o_p(1).
\label{eq:cross-fit-score}
\end{equation}
\end{proposition}

\begin{proof}
Under the bounded-temperature joint expansion,
\[
q_c
=
t_c
\widehat A_c^{(-k)}
(\widehat\mu_c^{(-k)})^\top z
+
O_p(\nu^{-1}).
\]
When clipping is inactive,
\[
t_c
\widehat A_c^{(-k)}
=
\frac{A_{\mathrm{ref}}}{\tau}.
\]
Consistency of the fitted direction then gives
Eq.~\eqref{eq:cross-fit-score}.
\end{proof}

The construction illustrates how gain equalization can be separated from
current-query self-inclusion at the estimator level.
The main experiments retain the original ProCo estimator and use
cross-fitting only as an analytical alternative.

\subsection{Empirical Finite-Sample Diagnostics}
\label{sec:finite-sample-diagnostics}

The supplementary support-size experiments keep the representation fixed
while changing only the number of observations used to estimate the class
state.

On CIFAR-100 Projection features, increasing support from
$n=5$ to $n=100$ substantially reduces resultant estimation error.
A separate frozen replay measures the score increment caused by including
the current query in its class statistics.

These diagnostics empirically separate support-induced estimation effects
from the population angular geometry represented by $A_c$.
\begin{figure}[H]
\centering
\includegraphics[width=85mm]{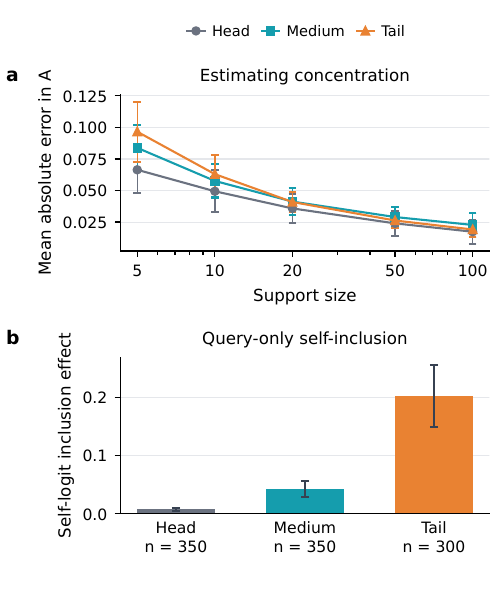}
\caption{\textbf{Support size and self-inclusion effects.} (a) Resultant-length estimation error; bars show between-class SD. (b) Current-query inclusion effects; bars show average within-class query SD.}
\label{fig:finite-sample}
\end{figure}

\section{Protocols for Angular-Response Controls}
\label{app:new-protocol}
\label{sec:controlled-protocols}

This section specifies the controlled score-realization experiments used to
separate angular gain from the other effects of classwise temperature
adjustment.

\subsection{Controlled Training Comparisons}
\label{sec:controlled-training-protocol}

We compare four score realizations:
the original ProCo score, direct classwise temperature adjustment,
intercept-preserving adjustment, and Pure Angular editing.
The target gain is
\[
g_c^\star
=
\bar g+\beta(g_c-\bar g),
\qquad
\bar g
=
\frac1C\sum_c g_c.
\]

Unless otherwise stated, complete equalization uses $\beta=0$. The native score corresponds to $\beta=1$. Within each controlled training comparison, all realizations use the same
training identities, initialization, data augmentations, optimizer, and
learning-rate schedule.
The online class-statistical state is retained from the corresponding ProCo
recipe.
Gain targets are computed from detached class statistics, and gradients do
not propagate through the statistical estimator or through the target-gain
construction.

For a training-frequency prior, the contrastive logits include $\log \pi_c$, where $\pi_c$ is computed from the corresponding long-tail training subset.
Uniform-prior controls instead use a common zero offset.
Contrastive-only experiments optimize only the probabilistic contrastive
objective.
Joint-training experiments additionally include the original classification
branch with unit weight.

The main representation-learning comparison uses the contrastive-only
objective with the training-frequency prior.
Additional uniform-prior and joint-training conditions are retained as
controlled extensions.

\subsection{Representation Evaluation}
\label{sec:representation-evaluation-protocol}

After the final training epoch, the encoder and projector are frozen.
Encoder and Projection representations are evaluated separately with ordinary
linear classifiers.

The main comparison uses the corresponding long-tail training subset as the
linear-classifier fitting pool.
The newly fitted classifiers use ordinary cross-entropy and do not include
an additional class-frequency logit offset or any vMF score modification.
Thus differences in the linear readouts arise only from the representations
learned during the preceding controlled training.

Historical balanced-support controls use the full balanced CIFAR fitting
pool and are reported separately.
The balanced-support and long-tail-support evaluations measure different
properties of the same frozen representation and are not pooled.

\subsection{Same-State Realization Diagnostic}
\label{sec:same-state-diagnostic}

The fixed-state diagnostic in the main text uses the completed
CIFAR-100-LT IF10 uniform-prior Direct Temperature model at its final state.

We take the first five queries of each class from the predefined feature
collection in its fixed ordering, giving 500 queries in total.
The query features, class statistics, and class offsets are held fixed while
only the score realization is changed.

Using the original ProCo score on this state gives $L_{\mathrm{CE}}=2.2452$ and diagnostic accuracy $37.2\%$.

Direct Temperature equalization gives $L_{\mathrm{CE}}=72.3403$ and assigns all 500 queries to a single class.

Intercept preservation gives $L_{\mathrm{CE}}=2.1502$, with $95$ distinct predicted classes.
Pure Angular gives $L_{\mathrm{CE}}=2.1503$, also with $95$ predicted classes.

For the collapsed class, $\kappa_c=5.4415, \; A_c=0.04244, \; \tau_c^\star=0.005730$. Its cosine-zero score changes from $q_c(0;\tau_0)=0.3888$ to $q_c(0;\tau_c^\star)=79.1456$. The exact cosine-zero angular slope under the adjusted temperature is $3.8058$, whereas the target leading gain is $7.4057$. Pure Angular preserves the native intercept and gives exact zero-cosine slope $7.4032$. This fixed-state comparison separates the score-level consequences of the
three realizations while leaving the representation and statistical state
unchanged.

\subsection{Frozen iNaturalist Realization Study}
\label{sec:inat-realization-protocol}

The frozen iNaturalist experiment uses the released ProCo representation.
Training identities are split once into an $80/20$ fitting/calibration split.
The fitting partition determines the class directions, resultants, and fitted
vMF concentrations.
The calibration partition supplies fixed queries for the local-gradient
diagnostics.

The official iNaturalist validation set is not used for fitting the class
statistics.
Encoder and Projection features are normalized before spherical scoring.

We evaluate the original score and the three controlled realizations at $\beta\in\{0,0.5\}$, under both zero offset and the training-frequency offset.

On Projection features, the original score obtains $63.072\%$ accuracy at zero offset and $49.169\%$ with the training-frequency offset.
Complete Pure Angular equalization gives $65.422\%$ and $49.226\%$, respectively.

The same intervention therefore acts differently under different additive
class offsets even when the representation and fitted geometric state are
unchanged.
\section{Experimental Protocols and Numerical Validation}
\label{app:modern-protocol}
\label{sec:frozen-numerical-protocols}

This section records the frozen-state, boundary, numerical-approximation,
and auxiliary architecture protocols used throughout the supplementary
experiments.

\subsection{Frozen Representations and Statistical States}
\label{sec:frozen-states}

Frozen analyses use deterministic features extracted from fixed ProCo
checkpoints.
Only training identities determine class directions and fitted
concentrations; held-out queries are never used to fit the class statistics.

CIFAR Projection features have dimension $128$.
ImageNet-LT and iNaturalist Projection features have dimension $1024$.
The released large-scale models use $2048$-dimensional Encoder features.

The finite-dimensional angular gain is
\[
g_c
=
\frac{A_p(\widehat\kappa_c)}{\tau_0},
\]
whereas the high-dimensional leading comparator uses
\[
g_c^{(0)}
=
\frac{r_0(\widehat\kappa_c/\nu)}{\tau_0}.
\]

Frozen rescoring holds the fitted state fixed.
This is distinct from the evolving class-statistical banks used during
training.

For CIFAR-100, ImageNet-LT, and iNaturalist, shot groups follow the original
training counts: $\text{head}:n_c>100, \; \text{medium}:20\le n_c\le 100, \; \text{tail}:n_c<20$. CIFAR-10 uses frequency-ranked groups because it contains only ten classes.

\subsection{Boundary Pair Selection}
\label{sec:boundary-selection}

For CIFAR-100-LT, ImageNet-LT, and iNaturalist, class pairs are selected from
training geometry before evaluation.

Within each of the six shot-group pair types $\mathrm{HH},\mathrm{HM},\mathrm{HT}, \mathrm{MM},\mathrm{MT},\mathrm{TT}$, we retain the five pairs with largest prototype cosine and five distinct
random control pairs.
This gives $60$ pairs per dataset.

The random construction uses the common experiment seed and does not inspect
evaluation labels, boundary errors, or intervention outcomes.

For CIFAR-10-LT IF100, all $\binom{10}{2}=45$ unordered class pairs are evaluated.

\subsection{Short-Arc Root Solving}
\label{sec:root-solving}

For prototypes $\mu_a$ and $\mu_b$, the short great-circle arc is evaluated
at $2001$ uniformly spaced locations.
Sign changes bracket candidate roots, and Brent refinement is used to obtain
the final root.

A displacement is reported only when both the original and intervened scores
have a unique root on the short arc.
Root creation, root disappearance, and multiple-root cases are recorded
separately.

At zero prior, the finite-dimensional linear-mean comparator gives boundary
displacement MAEs of $0.11986^\circ,\; 0.00272^\circ,\; 0.00419^\circ$ on CIFAR-100, ImageNet-LT, and iNaturalist, respectively.

For CIFAR-10, all $45$ class pairs are comparable and the displacement MAE is $0.011627^\circ$, with maximum error $0.029000^\circ$. \subsection{Pair--Query Margin Comparisons}
\label{sec:margin-change-protocol}

Queries from both classes of every selected pair are retained.
A query can therefore appear in more than one pair; reported denominators are
pair--query incidences rather than independent query counts.

Writing the measured intervention-induced margin change as
\[
\Delta m
=
\Delta m_{\mathrm{lin}}+r,
\]
the condition
\[
|\Delta m_{\mathrm{lin}}|
>
|r|
\]
certifies the sign of the exact margin change.

Using the measured full-score residual, this certificate covers $97.57\%,\; 99.83\%,\; 100\%$ of the selected incidences on CIFAR-100-LT, ImageNet-LT, and iNaturalist,
respectively, and is correct on every covered incidence.

\subsection{Numerical Expansion Validation}
\label{sec:numerical-expansion-validation}

The score-expansion experiments use independent high-precision Bessel
evaluations as reference values.

Residual-order fits require at least three points and coefficient of
determination $R^2\ge0.9$. Successive joint approximations recover median residual slopes $-0.999,\; -2.000,\; -2.999$. The second-corrected joint score has median absolute error $2.32\times 10^{-7}$ on the prescribed joint-regime grid.

The fixed-dimensional approximation is evaluated on its own
low-dimensional, high-concentration regime rather than treated as a universal
competitor.

\subsection{Local Gradient Diagnostics}
\label{sec:gradient-protocol}

For local-gradient comparisons, the model and its class-statistical state are
frozen.
Different score policies are then applied to identical queries.

The contrastive query gradient is projected by $\Pi_z=I-zz^\top$ onto the tangent space of the normalized representation.

For cosine-versus-Pure Angular comparisons, $1024$ held-out identities are
selected per frozen state without reference to labels or outcomes.
Both scores use identical class offsets and identical fitted class
statistics.

The reported stacked-query relative error is
\[
100\,
\frac{
\|G_{\mathrm{angular}}-G_{\cos}\|_F
}{
\|G_{\mathrm{angular}}\|_F
}.
\]

\subsection{Support-Size and Distribution-Fit Diagnostics}
\label{sec:support-fit-protocol}

Support-size experiments keep the representation fixed and vary only the
number of observations used to estimate the class state.

A disjoint finite reference set is used to measure resultant-estimation
error.
The reference is not interpreted as a known population parameter.

A separate frozen replay compares score changes from excluding the current
query with those from excluding the query identity's stored history.

These diagnostics separate goodness of fit of the population model from
predictive accuracy of the fitted score geometry.
Single-vMF fit diagnostics use held-out features and are kept distinct from
the approximation error of the fitted vMF score itself.

\begin{figure}[H]
  \centering
  \includegraphics[width=0.48\linewidth]{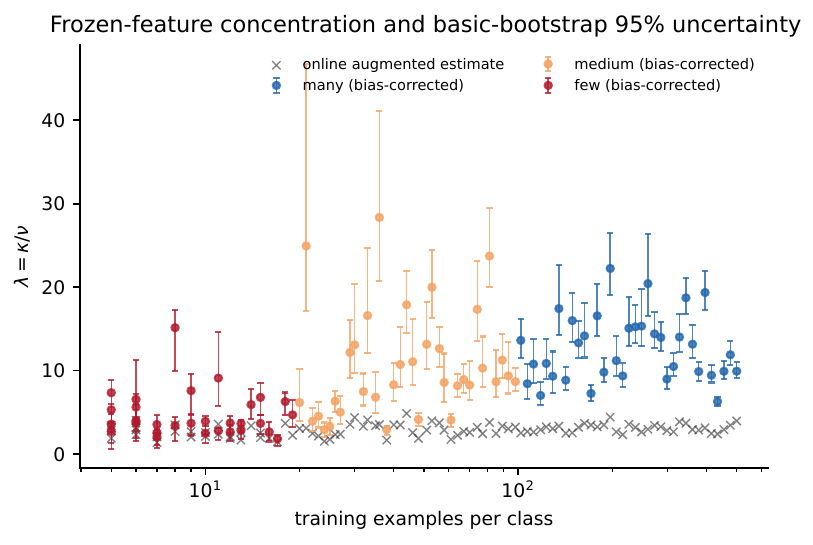}\hfill
  \includegraphics[width=0.48\linewidth]{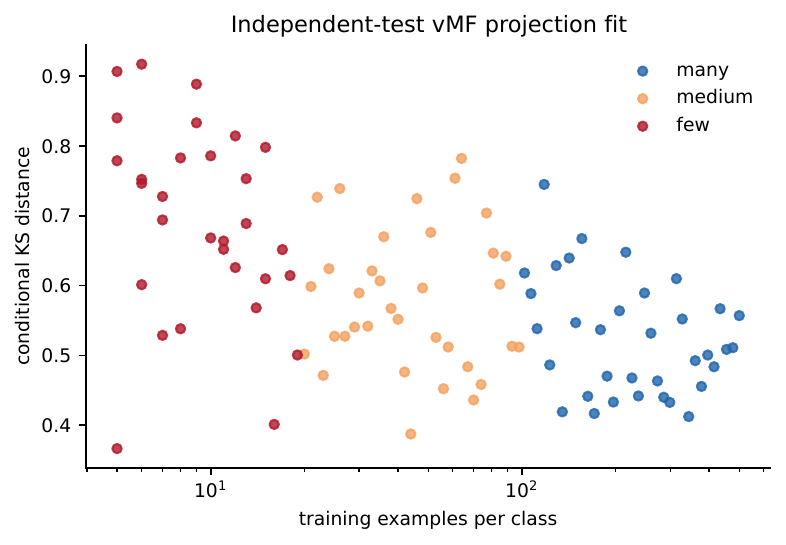}
  \caption{\textbf{Class geometry and vMF goodness of fit.} Left: class frequency and fitted concentration. Right: held-out projection goodness of fit.}
  \label{fig:real-scope-app}
\end{figure}

\begin{figure}[H]
  \centering
  \includegraphics[width=\linewidth]{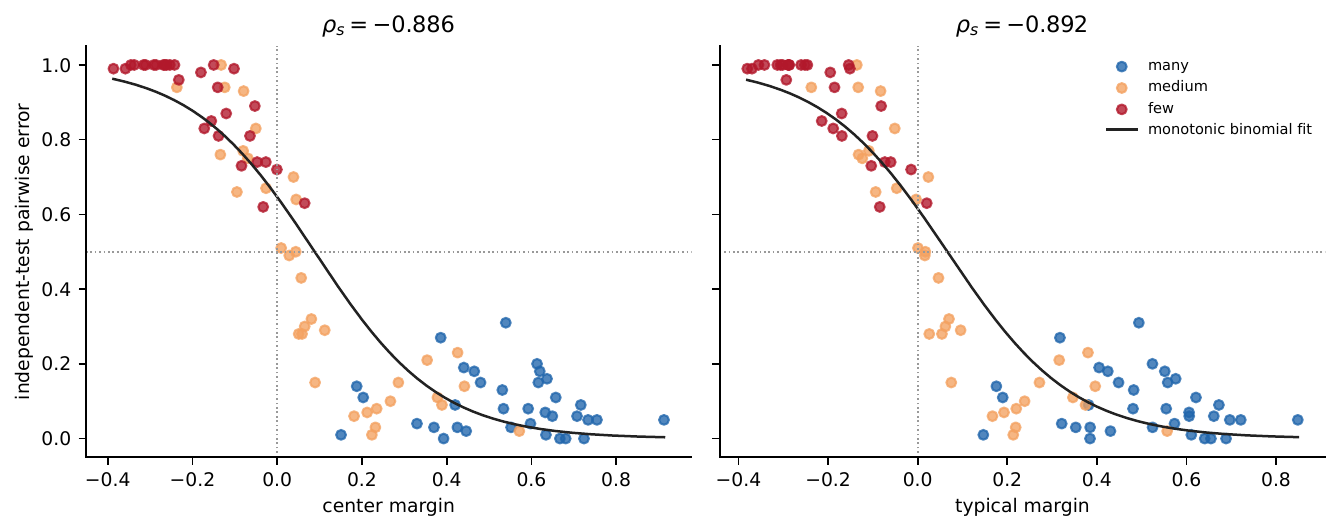}
  \caption{\textbf{Margins and fitted-score errors.} Prototype-centre and random-typical margins versus reconstructed pairwise errors.}
  \label{fig:real-margin-app}
\end{figure}

\subsection{Matched ResNet and ResNeXt Comparison}
\label{sec:backbone-control}
\label{app:cross-backbone}

We apply the identical frozen protocol to the public 90-epoch ImageNet-LT
ProCo ResNet-50 and ResNeXt-50 checkpoints
\citep{du2024probabilisticcontrastivelearninglongtailed,he2016deep,Xie_2017_CVPR}.
All 115,846 training features determine class statistics, and all 20,000
validation images are queried. Both models use a 2,048-dimensional encoder
and a 1,024-dimensional normalized projection output. The full fixed grid
uses $\beta\in\{0,0.5,1,1.5\}$, $\gamma\in\{-1,0,1\}$, and $\tau_0=0.1$.
Table~\ref{tab:cross-backbone} reports the projection results. Equalization
at zero prior changes accuracy by $+0.195$ and $+0.315$ percentage points,
whereas at $\gamma=-1$ it changes accuracy by $-1.665$ and $-1.465$ points.
Thus prior-dependent gain sensitivity is present in both architectures.
In the encoder space, zero-prior equalization gives $+2.020$ and $+1.635$
points under the same fitted-score protocol.

The released ResNeXt classification head separately obtains $58.026\%$
on the 50,000-image test split, with head/medium/tail accuracies of
$67.787/55.620/38.868\%$. Public checkpoints support this frozen comparison;
they are not the newly trained models in Table~\ref{tab:historical-power-temperature}.

\begin{table}[ht]
\centering\small
\caption{\textbf{Frozen gain interventions across backbones.} ImageNet-LT projection-score accuracy changes in percentage points relative to ProCo ($\beta=1$) at the same prior.}
\label{tab:cross-backbone}
\begin{tabular}{@{}rrrr@{}}
\toprule
$\beta$ & $\gamma$ & ResNet-50 & ResNeXt-50 \\
\midrule
0 & -1 & -1.665 & -1.465 \\
0 & 0 & +0.195 & +0.315 \\
0 & 1 & +0.175 & -0.175 \\
0.5 & -1 & -0.635 & -0.610 \\
0.5 & 0 & +0.260 & +0.240 \\
0.5 & 1 & +0.055 & -0.040 \\
1 & -1 & +0.000 & +0.000 \\
1 & 0 & +0.000 & +0.000 \\
1 & 1 & +0.000 & +0.000 \\
1.5 & -1 & +0.650 & +0.625 \\
1.5 & 0 & -0.395 & -0.280 \\
1.5 & 1 & -0.095 & -0.050 \\
\bottomrule
\end{tabular}
\end{table}

\section{Additional Mechanism Studies}
\label{sec:additional-mechanism-studies}

This section reports auxiliary experiments that probe the mechanism outside
the main Pure Angular training comparison.

\subsection{Prior-Dependent Frozen Gain Interventions}
\label{sec:prior-dependent-frozen}

We first vary gain disparity and additive class offsets on fixed
representations.

The frozen gain path is
\[
g_c^{\mathrm{lin}}(\beta)
=
\bar g+\beta(g_c-\bar g),
\]
with $\beta\in\{0,0.5,1,1.5\}$. The class-offset coefficient is varied over $\gamma\in\{-1,0,1\}$. At a fixed prior, $\beta=1$ is the original ProCo-gain reference.

At zero offset, complete equalization changes Projection accuracy by $+2.57$ points on CIFAR-100-LT, $+0.195$ on ImageNet-LT, and $+2.297$ on iNaturalist.

Under $\gamma=-1$, the corresponding CIFAR-100 and ImageNet changes become $-5.35$ and $-1.665$ points.

The same gain manipulation therefore interacts strongly with the additive
class offset.

\begin{figure}[H]
\centering
\includegraphics[width=85mm]{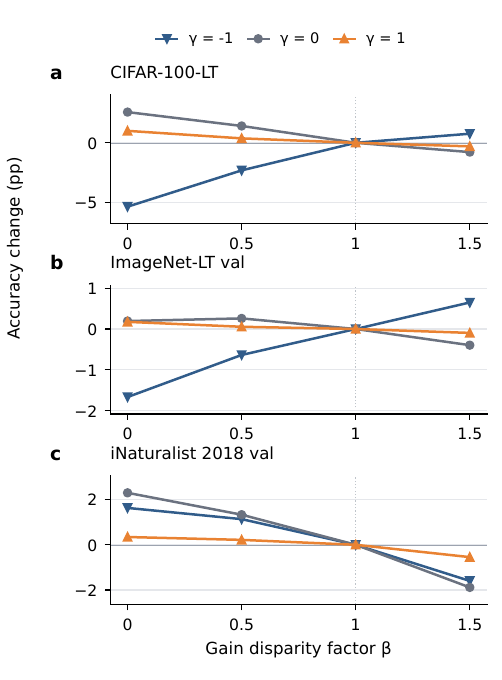}
\caption{\textbf{Prior-dependent effects of angular gain.} Frozen vMF accuracy changes relative to ProCo at the same prior strength $\gamma$.}
\label{fig:gain-prior}
\end{figure}

\subsection{Historical Power-Temperature Intervention}
\label{sec:historical-power-temperature}

Before introducing the cleaner linear Pure Angular control, we also examined
a positive power-temperature family during training:
\[
g_c^{\mathrm{pow}}(\beta)
=
\bar g\,
\frac{
g_c^\beta
}{
C^{-1}\sum_j g_j^\beta
},
\]
with
\[
\tau_c
=
\frac{A_c}{g_c^{\mathrm{pow}}(\beta)},
\qquad
\beta\in\{0.5,1,1.5\}.
\]

This path preserves the mean applied gain while reducing or increasing its
between-class dispersion.
Because it acts through temperature, it changes the complete score rather
than only the linear angular term.

The historical study uses CIFAR-10-LT and CIFAR-100-LT at IF50 and IF100
with ResNet-32 and joint training.

For CIFAR-100-LT IF100, the final applied gain standard deviations are $0.630,\; 0.349,\; 0.948$ for ProCo, reduced disparity, and increased disparity.

The corresponding native statistical-state gain standard deviations are $0.630,\; 0.684,\; 0.648$. Thus changing the applied gain does not trivially impose the same ordering on
the learned class-statistical state.

\subsection{Local Gradient Prediction During Training}
\label{sec:historical-gradient}

At selected epochs, each historical model and its statistical bank are
frozen.
All three gain policies are then evaluated on identical queries.

At epoch $200$, the leading-gradient relative errors under the respective
training policies are $4.87\%,\; 5.34\%,\; 4.57\%$. For the same-state change from the native policy to either gain intervention,
the leading angular approximation predicts the exact gradient increment with
mean relative error $8.82\%-10.18\%$. Thus the leading gain explains most of the local gradient response while
leaving measurable finite-dimensional corrections.

\begin{figure}[H]
\centering
\includegraphics[width=85mm]{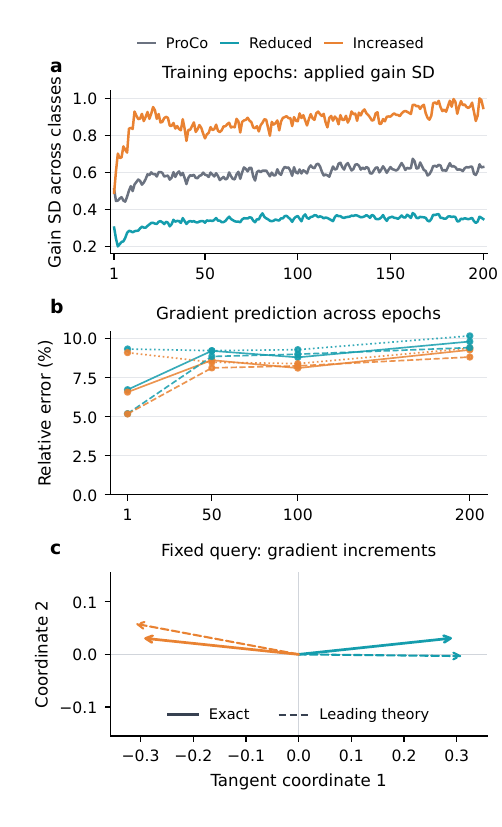}
\caption{\textbf{Angular gain and local gradients.} (a) Gain dispersion during training. (b) Relative error of leading gradient increments. Colors identify interventions; line styles identify training states. (c) Exact (solid) and leading (dashed) increments projected onto a common plane.}
\label{fig:training-gain}
\end{figure}

\subsection{Historical Classification Outcomes}
\label{sec:historical-classification}

The original jointly trained classification heads at the final epoch are:

\begin{table}[t]
\centering
\small
\caption{
Historical power-temperature intervention.
Original classification-head test accuracy (\%).
}
\label{tab:historical-power-temperature}
\begin{tabular}{lcccc}
\toprule
& \multicolumn{2}{c}{CIFAR-10-LT}
& \multicolumn{2}{c}{CIFAR-100-LT}\\
Policy & IF50 & IF100 & IF50 & IF100\\
\midrule
ProCo & 87.57 & 85.39 & 55.78 & 51.34\\
Reduced gain disparity & 87.72 & 85.10 & 56.62 & 51.72\\
Increased gain disparity & 87.56 & 85.14 & 55.99 & 52.05\\
\bottomrule
\end{tabular}
\end{table}

The relative ordering changes across imbalance settings.
This historical family therefore supports gain as an active training
variable but does not define a monotone rule for reducing or increasing
gain disparity.

\begin{figure}[H]
\centering
\includegraphics[width=85mm]{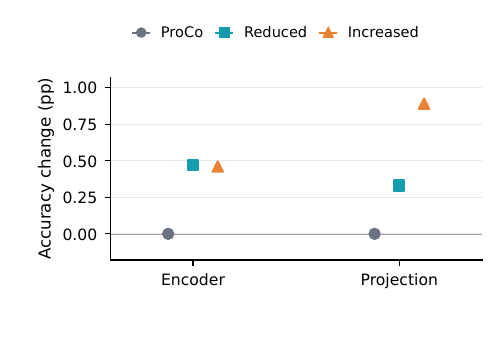}
\caption{\textbf{Changes in common linear readouts.} Accuracy differences
from ProCo at epoch 200 within each feature layer on CIFAR-100-LT IF100.
Each classifier is fitted independently under the same fixed protocol.
These probes are distinct from the original classification heads in
Table~\ref{tab:historical-power-temperature}.}
\label{fig:common-readouts}
\end{figure}

\section{Complete CIFAR-10 Realization Controls}
\label{sec:cifar10-controls}

This section reports the complete CIFAR-10 control matrix across imbalance,
training prior, training objective, score realization, and intervention
strength.

All models use the same controlled two-view recipe within a setting.
The score policies are ProCo and the three realizations
Direct Temperature, Intercept-Preserving, and Pure Angular at $\beta\in\{0,0.5\}$. Uniform-prior runs use zero class offset.
Training-frequency-prior runs use the class counts of the corresponding
long-tail subset.

Encoder (E) and Projection (P) denote balanced-support linear probes.
FC denotes the original classification head and is available only under
joint training.

\subsection{Complete Numerical Results}

\begin{table}[H]
\centering
\small
\caption{CIFAR-10, IF1, uniform prior. Accuracy (\%).}
\label{tab:c10-if1-u}
\begin{tabular}{lc|cc|ccc}
\toprule
& & \multicolumn{2}{c|}{Contrastive-only}
& \multicolumn{3}{c}{Joint}\\
Realization & $\beta$ & E & P & E & P & FC\\
\midrule
ProCo & 1 & 92.15 & 92.91 & 92.95 & 93.04 & 92.81\\
Temperature & 0 & 92.09 & 93.13 & 92.77 & 92.60 & 92.54\\
Temperature & 0.5 & 92.32 & 92.95 & 92.40 & 92.32 & 92.38\\
Intercept-preserving & 0 & 92.46 & 93.01 & 93.01 & 93.04 & 92.85\\
Intercept-preserving & 0.5 & 92.59 & 93.19 & 92.84 & 92.89 & 92.55\\
Pure Angular & 0 & 92.13 & 93.05 & 93.06 & 93.12 & 92.87\\
Pure Angular & 0.5 & 92.62 & 93.11 & 92.60 & 92.38 & 92.27\\
\bottomrule
\end{tabular}
\end{table}

\begin{table}[H]
\centering
\small
\caption{CIFAR-10, IF10, uniform prior. Accuracy (\%).}
\label{tab:c10-if10-u}
\begin{tabular}{lc|cc|ccc}
\toprule
& & \multicolumn{2}{c|}{Contrastive-only}
& \multicolumn{3}{c}{Joint}\\
Realization & $\beta$ & E & P & E & P & FC\\
\midrule
ProCo & 1 & 88.97 & 88.63 & 88.57 & 88.28 & 86.32\\
Temperature & 0 & 88.34 & 87.58 & 88.66 & 88.01 & 85.96\\
Temperature & 0.5 & 88.29 & 88.21 & 88.20 & 87.84 & 85.96\\
Intercept-preserving & 0 & 88.61 & 88.20 & 88.36 & 87.96 & 85.62\\
Intercept-preserving & 0.5 & 88.51 & 88.24 & 88.34 & 87.96 & 85.83\\
Pure Angular & 0 & 88.70 & 88.83 & 88.00 & 87.54 & 86.07\\
Pure Angular & 0.5 & 88.73 & 88.87 & 88.48 & 88.11 & 86.35\\
\bottomrule
\end{tabular}
\end{table}

\begin{table}[H]
\centering
\small
\caption{CIFAR-10, IF10, training-frequency prior. Accuracy (\%).}
\label{tab:c10-if10-f}
\begin{tabular}{lc|cc|ccc}
\toprule
& & \multicolumn{2}{c|}{Contrastive-only}
& \multicolumn{3}{c}{Joint}\\
Realization & $\beta$ & E & P & E & P & FC\\
\midrule
ProCo & 1 & 88.60 & 88.41 & 88.44 & 88.19 & 87.46\\
Temperature & 0 & 88.60 & 88.94 & 88.23 & 88.00 & 87.29\\
Temperature & 0.5 & 88.55 & 88.34 & 88.31 & 88.10 & 87.34\\
Intercept-preserving & 0 & 88.74 & 89.02 & 88.23 & 87.96 & 87.41\\
Intercept-preserving & 0.5 & 88.35 & 88.53 & 88.44 & 87.76 & 87.27\\
Pure Angular & 0 & 88.98 & 89.19 & 88.49 & 87.96 & 87.24\\
Pure Angular & 0.5 & 88.58 & 88.24 & 88.11 & 87.79 & 87.43\\
\bottomrule
\end{tabular}
\end{table}

\begin{table}[H]
\centering
\small
\caption{CIFAR-10, IF50, uniform prior. Accuracy (\%).}
\label{tab:c10-if50-u}
\begin{tabular}{lc|cc|ccc}
\toprule
& & \multicolumn{2}{c|}{Contrastive-only}
& \multicolumn{3}{c}{Joint}\\
Realization & $\beta$ & E & P & E & P & FC\\
\midrule
ProCo & 1 & 82.70 & 81.52 & 83.70 & 82.29 & 76.14\\
Temperature & 0 & 76.82 & 68.89 & 83.63 & 81.89 & 75.81\\
Temperature & 0.5 & 80.98 & 78.34 & 83.97 & 82.07 & 76.04\\
Intercept-preserving & 0 & 78.18 & 72.34 & 83.78 & 81.73 & 75.23\\
Intercept-preserving & 0.5 & 81.54 & 79.21 & 83.36 & 81.88 & 76.28\\
Pure Angular & 0 & 77.64 & 70.61 & 84.09 & 81.95 & 76.30\\
Pure Angular & 0.5 & 81.55 & 79.55 & 84.24 & 82.10 & 75.52\\
\bottomrule
\end{tabular}
\end{table}

\begin{table}[H]
\centering
\small
\caption{CIFAR-10, IF50, training-frequency prior. Accuracy (\%).}
\label{tab:c10-if50-f}
\begin{tabular}{lc|cc|ccc}
\toprule
& & \multicolumn{2}{c|}{Contrastive-only}
& \multicolumn{3}{c}{Joint}\\
Realization & $\beta$ & E & P & E & P & FC\\
\midrule
ProCo & 1 & 84.86 & 83.91 & 83.42 & 82.04 & 80.52\\
Temperature & 0 & 84.57 & 83.61 & 83.36 & 82.95 & 80.46\\
Temperature & 0.5 & 85.16 & 84.31 & 84.25 & 83.37 & 81.24\\
Intercept-preserving & 0 & 84.60 & 84.19 & 84.13 & 83.07 & 80.81\\
Intercept-preserving & 0.5 & 84.38 & 84.04 & 84.60 & 82.93 & 80.79\\
Pure Angular & 0 & 85.05 & 84.26 & 83.47 & 82.65 & 80.79\\
Pure Angular & 0.5 & 84.78 & 83.90 & 84.25 & 82.71 & 80.80\\
\bottomrule
\end{tabular}
\end{table}

\begin{table}[H]
\centering
\small
\caption{CIFAR-10, IF100, uniform prior. Accuracy (\%).}
\label{tab:c10-if100-u}
\begin{tabular}{lc|cc|ccc}
\toprule
& & \multicolumn{2}{c|}{Contrastive-only}
& \multicolumn{3}{c}{Joint}\\
Realization & $\beta$ & E & P & E & P & FC\\
\midrule
ProCo & 1 & 78.28 & 74.65 & 81.64 & 79.55 & 71.25\\
Temperature & 0 & 70.46 & 57.88 & 81.10 & 76.62 & 71.58\\
Temperature & 0.5 & 74.56 & 67.21 & 81.93 & 79.16 & 71.28\\
Intercept-preserving & 0 & 70.65 & 58.17 & 81.44 & 79.14 & 71.44\\
Intercept-preserving & 0.5 & 74.82 & 67.63 & 81.42 & 78.43 & 71.31\\
Pure Angular & 0 & 70.18 & 54.81 & 81.39 & 78.40 & 71.70\\
Pure Angular & 0.5 & 75.41 & 67.30 & 81.57 & 78.30 & 71.43\\
\bottomrule
\end{tabular}
\end{table}

\begin{table}[H]
\centering
\small
\caption{CIFAR-10, IF100, training-frequency prior. Accuracy (\%).}
\label{tab:c10-if100-f}
\begin{tabular}{lc|cc|ccc}
\toprule
& & \multicolumn{2}{c|}{Contrastive-only}
& \multicolumn{3}{c}{Joint}\\
Realization & $\beta$ & E & P & E & P & FC\\
\midrule
ProCo & 1 & 82.49 & 81.32 & 81.50 & 80.13 & 77.79\\
Temperature & 0 & 82.43 & 80.56 & 81.56 & 79.72 & 77.39\\
Temperature & 0.5 & 82.15 & 80.41 & 81.29 & 79.52 & 77.11\\
Intercept-preserving & 0 & 82.64 & 81.49 & 80.90 & 79.31 & 76.71\\
Intercept-preserving & 0.5 & 82.71 & 81.53 & 81.12 & 78.91 & 76.36\\
Pure Angular & 0 & 82.51 & 81.19 & 80.89 & 79.07 & 76.50\\
Pure Angular & 0.5 & 82.86 & 81.46 & 80.89 & 79.31 & 77.37\\
\bottomrule
\end{tabular}
\end{table}

\subsection{Cross-Setting Pattern}
\label{sec:cifar10-cross-setting}

The complete matrix shows a strong interaction between gain editing and the
training prior.

Under a uniform prior, complete Pure Angular equalization substantially
reduces contrastive-only representation accuracy at IF50 and IF100.
With the training-frequency prior, the same intervention has much smaller
effects and is positive in the main long-tail evaluation reported in the
paper.

Joint training introduces an additional objective dependence.
At IF50 with the training-frequency prior, partial editing improves all three
joint readouts.
At IF100, both Pure Angular strengths reduce the joint readouts.

These results locate the source of the mixed extension results:
the effect of angular gain depends on the additive prior, intervention
strength, and training objective.

\clearpage
\begin{figure}[H]
\centering
\includegraphics[width=\linewidth]{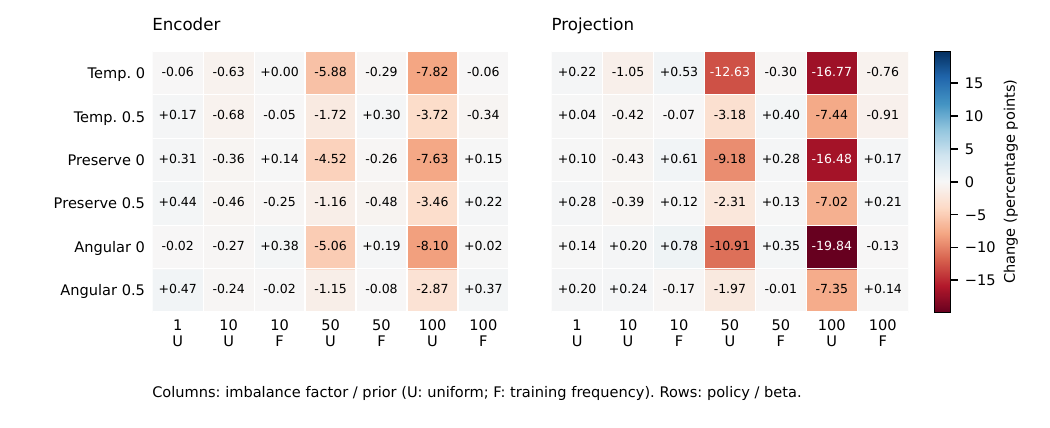}
\caption{\textbf{Contrastive-only training across sampling and prior settings.} Balanced-support probe accuracy changes from ProCo in percentage points; colors are centered at zero.}
\label{fig:c10-pure-complete}
\end{figure}
\clearpage

\begin{figure}[H]
\centering
\includegraphics[width=\linewidth]{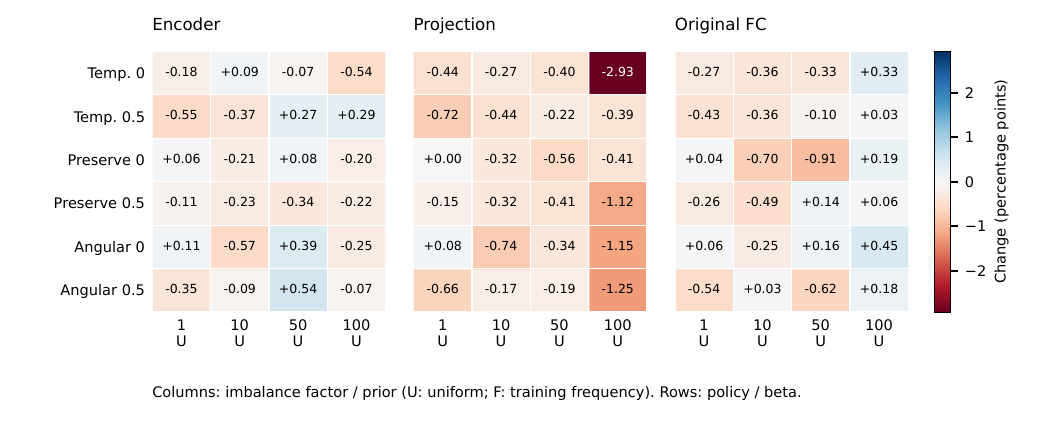}
\caption{\textbf{Joint training with a uniform prior.} Accuracy changes from ProCo in percentage points. Encoder/Projection use balanced-support probes; FC is the jointly trained head.}
\label{fig:c10-joint-uniform-complete}
\end{figure}

\begin{figure}[H]
\centering
\includegraphics[width=\linewidth]{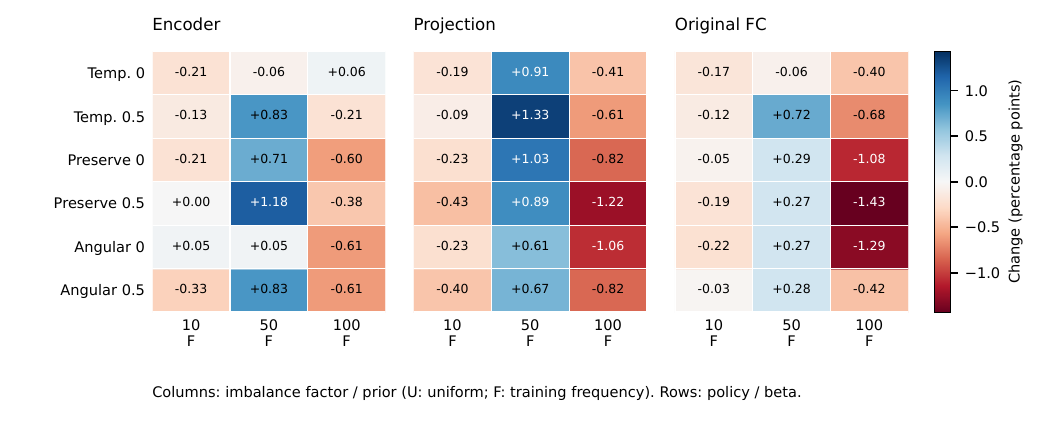}
\caption{\textbf{Joint training with the training-frequency prior.} Accuracy changes from ProCo in percentage points. Encoder/Projection use balanced-support probes; FC is the jointly trained head.}
\label{fig:c10-joint-actual-complete}
\end{figure}

\subsection{Local Mechanism Across Training States}
\label{sec:cifar10-local-mechanism}

For the Pure Angular and Intercept-Preserving policies, the cosine-zero
intercept is preserved at every recorded training state.
Consequently, accuracy decreases under these controls cannot be attributed to
intercept drift.

At the recorded states, the leading angular gradient remains close to the
exact gradient even in settings whose final representation accuracy is lower.
The local score mechanism and the final optimization outcome are therefore
distinct empirical objects.
\begin{figure}[H]
\centering
\includegraphics[width=\linewidth]{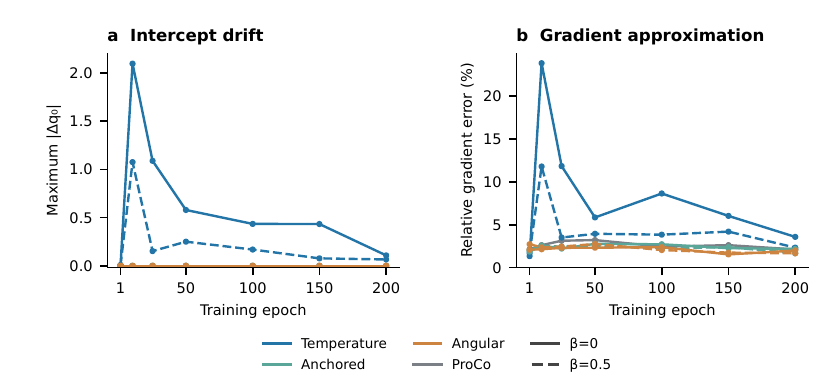}
\caption{\textbf{Score and gradient trajectories during training.} IF50 with a uniform prior: cosine-zero intercept drift and leading-gradient approximation error.}
\label{fig:c10-trajectory}
\end{figure}

\section{Representation-Learning Extensions}
\label{app:learning-extensions}
\label{sec:learning-extensions}

This section collects representation-learning results outside the main
contrastive-only, training-frequency-prior comparison.

\subsection{Joint-Training Classification Heads}
\label{sec:joint-head-results}

Joint training adds the original classification objective to the
probabilistic contrastive objective.
The intervention can therefore affect a shared encoder through both
branches.

The original classification-head accuracies under complete equalization are
shown below.

\begin{table}[H]
\centering
\small
\caption{
Joint-training FC-head accuracy (\%) with the training-frequency prior.
All modified policies use $\beta=0$.
$\Delta$ is Pure Angular minus ProCo.
}
\label{tab:joint-fc}
\begin{tabular}{lrrrrr}
\toprule
Dataset / IF & ProCo & Direct Temp. & Preserve & Pure Angular & $\Delta$\\
\midrule
CIFAR-10 / 10  & 87.46 & 87.29 & 87.41 & 87.24 & -0.22\\
CIFAR-10 / 50  & 80.52 & 80.46 & 80.81 & 80.79 & +0.27\\
CIFAR-10 / 100 & 77.79 & 77.39 & 76.71 & 76.50 & -1.29\\
CIFAR-100 / 10  & 58.26 & 58.24 & 57.55 & 58.75 & +0.49\\
CIFAR-100 / 50  & 47.63 & 46.73 & 48.82 & 48.10 & +0.47\\
CIFAR-100 / 100 & 44.67 & 43.15 & 43.54 & 43.88 & -0.79\\
ImageNet-LT      & 57.26 & 56.75 & -- & 57.35 & +0.09\\
\bottomrule
\end{tabular}
\end{table}

The response is mixed:
Pure Angular improves CIFAR-100 at IF10/50 and CIFAR-10 at IF50, but declines
at the remaining CIFAR settings.
This contrasts with the consistent gains in the main contrastive-only
training comparison.

\subsection{Training Prior and Evaluation Support}
\label{sec:prior-support-extension}

The main representation table uses the training-frequency prior during
feature learning and the corresponding long-tail subset for linear
evaluation.

Uniform-prior training leaves the examples long-tailed but removes the
frequency-dependent logit offset.

Under this control, complete Pure Angular equalization on CIFAR-10 changes
long-tail Encoder/Projection accuracy by $-11.47/-14.51$ points at IF50 and $-15.98/-18.31$ points at IF100.

Balanced-support probes use the full balanced CIFAR fitting pool and therefore
measure a different downstream labeled-support regime.
Absolute accuracies from the balanced-support and long-tail-support protocols
are not pooled.

\subsection{Frozen Decision Rules Versus Learned Representations}
\label{sec:frozen-vs-learning}

The frozen vMF experiments and representation-learning experiments answer
different questions.

Frozen scoring holds the representation fixed and changes only the decision
rule.
For example, on CIFAR-10 IF50/100, complete Pure Angular frozen scoring
changes Encoder/Projection accuracy by $-0.75/-0.14$ and $-0.74/-0.23$ points, while retaining $99.46\%-99.95\%$ prediction agreement with cosine prototypes.

By contrast, the main linear-probe results compare representations learned
under different training score realizations.

These two evaluations are therefore used as complementary evidence:
the frozen experiments identify the score geometry, while the end-to-end
training experiments test whether the intervention changes learned features.

\subsection{ImageNet Gain-Strength Extension}
\label{sec:imagenet-strength-extension}

In the three-seed ImageNet-LT contrastive-only comparison, original ProCo
gives $23.05\pm0.12\%$ and $36.86\pm0.06\%$ Encoder/Projection accuracy,
while complete Pure Angular equalization gives $23.22\pm0.19\%$ and
$37.07\pm0.11\%$, respectively (Table~\ref{tab:angular-training}).
A separate single-run $\beta=1.5$ extension gives $23.56/37.76\%$.
Thus gain editing remains active at ImageNet scale, while complete
equalization is not the best tested gain profile in this setting.

\subsection{Backbone and Calibration Controls}
\label{sec:backbone-calibration-extension}

The frozen ResNet-50/ResNeXt-50 comparison in
Appendix~\ref{app:cross-backbone}
shows qualitatively similar prior dependence across architectures.

Accuracy and probabilistic calibration can also move in different directions.
In the original iNaturalist zero-offset frozen study, equalization improves
Encoder/Projection accuracy by approximately $7.165/2.297$ points while increasing NLL by $0.0521/0.0097$. These results separate classification geometry from probability calibration.
\section{Algorithm and Hyperparameter Summary}
\label{sec:algorithm-hyperparameters}

This section summarizes the score intervention, training procedure, and the
two evaluation procedures used in the paper.

\subsection{Controlled Score Realization and Training}
\label{sec:training-algorithm}

\begin{algorithm}[H]
\caption{Controlled score realization and representation learning}
\label{alg:controlled-training}
\begin{algorithmic}[1]
\REQUIRE
long-tail training set;
encoder $f_\theta$;
projector $h_\phi$;
optional classification head $W$;
baseline statistical estimator;
base temperature $\tau_0$;
strength $\beta$;
realization
$\mathcal R\in
\{\mathrm{native},\mathrm{temp},\mathrm{preserve},\mathrm{angular}\}$;
prior coefficient $\gamma$;
classification weight $\alpha\in\{0,1\}$.

\STATE Compute class frequencies
\[
\pi_c=\frac{n_c}{\sum_j n_j}.
\]

\FOR{each mini-batch}
    \STATE Generate the baseline augmented views and compute normalized
    Projection features
    \[
    z=
    \frac{h_\phi(f_\theta(x))}
    {\|h_\phi(f_\theta(x))\|}.
    \]

    \STATE Apply the baseline statistical-estimator update/order and obtain
    the class snapshot used for scoring.
    Detach $(\mu_c,\kappa_c)$ from the computation graph.

    \STATE Compute
    \[
    A_c=
    \frac{I_{p/2}(\kappa_c)}
    {I_{p/2-1}(\kappa_c)},
    \qquad
    g_c=\frac{A_c}{\tau_0},
    \qquad
    \bar g=\frac1C\sum_c g_c,
    \]
    and
    \[
    g_c^\star
    =
    \bar g+\beta(g_c-\bar g).
    \]

    \STATE For each query compute
    \[
    \rho_c=\mu_c^\top z,
    \qquad
    q_c^0=q_c(\rho_c;\tau_0).
    \]

    \STATE Choose the score realization
    \[
    q_c^{\mathcal R}
    =
    \begin{cases}
    q_c^0,
    &\mathcal R=\mathrm{native},\\[3pt]
    q_c(\rho_c;\tau_c^\star),
    &\mathcal R=\mathrm{temp},\\[3pt]
    q_c(\rho_c;\tau_c^\star)
    -q_c(0;\tau_c^\star)
    +q_c(0;\tau_0),
    &\mathcal R=\mathrm{preserve},\\[3pt]
    q_c^0+(g_c^\star-g_c)\rho_c,
    &\mathcal R=\mathrm{angular},
    \end{cases}
    \]
    where
    \[
    \tau_c^\star
    =
    \frac{A_c}{g_c^\star}.
    \]

    \STATE Form the contrastive logits
    \[
    s_c
    =
    q_c^{\mathcal R}
    +
    \gamma\log\pi_c
    \]
    and compute the probabilistic contrastive cross-entropy
    $L_{\mathrm{ProCo}}$.

    \IF{$\alpha=1$}
        \STATE Compute the baseline classification-head loss
        $L_{\mathrm{cls}}$ and set
        \[
        L
        =
        L_{\mathrm{ProCo}}
        +
        L_{\mathrm{cls}}.
        \]
    \ELSE
        \STATE Set
        \[
        L=L_{\mathrm{ProCo}}.
        \]
    \ENDIF

    \STATE Backpropagate through the score to the query representation,
    encoder, projector, and optional classification head.
    Do not differentiate through the statistical snapshot or the target-gain
    construction.

    \STATE Update trainable parameters using the baseline optimizer and
    schedule.
\ENDFOR

\STATE Save the fixed final-epoch model.
\end{algorithmic}
\end{algorithm}

The estimator snapshot follows the update order of the corresponding baseline
implementation.
The main intervention does not replace the statistical estimator.

For CIFAR, two strong augmented views supply the contrastive objective.
For ImageNet, the forward construction retains one weak and two strong views;
the two strong views supply the probabilistic contrastive objective.
In contrastive-only training, the classification loss is omitted.

\subsection{Ordinary Linear Evaluation}
\label{sec:linear-evaluation-algorithm}

\paragraph{Algorithm 2A: representation evaluation.}
For a trained model:

\begin{enumerate}
    \item Freeze the encoder and projector.
    Extract deterministic Encoder and Projection features using the
    preprocessing specified for the corresponding dataset.

    \item Fit one ordinary affine classifier on Encoder features and one on
    Projection features using the prescribed labeled fitting pool.
    Optimize only the new classifier parameters with standard
    cross-entropy.

    \item Do not apply a vMF score, score intervention, or additional
    class-frequency logit offset during this linear evaluation.

    \item Evaluate the classifiers on the held-out split and compare
    representations learned under different training score realizations.
\end{enumerate}

\subsection{Direct Frozen vMF Evaluation}
\label{sec:frozen-vmf-algorithm}

\paragraph{Algorithm 2B: frozen decision-rule evaluation.}
For a fixed checkpoint:

\begin{enumerate}
    \item Extract unit-normalized features in the chosen Encoder or
    Projection space.

    \item Using only the fitting images, estimate $(\mu_c,\kappa_c,A_c)$ for each class.

    \item Hold the representation and all fitted statistics fixed.
    Apply the requested score realization to every held-out query.

    \item For the main frozen comparison, use zero evaluation offset and
    predict
    \[
    \widehat y
    =
    \arg\max_c q_c^{\mathcal R}.
    \]

    \item Report accuracy and prediction agreement on the identical queries.
\end{enumerate}

Algorithm 2A compares representations learned under different objectives.
Algorithm 2B compares decision rules on the same representation.

\subsection{Feature-Learning Hyperparameters}
\label{sec:training-hyperparameters}

\begin{table}[H]
\centering
\small
\caption{Main feature-learning hyperparameters.}
\label{tab:training-hyperparameters}
\begin{tabular}{lll}
\toprule
Setting & CIFAR-10/100-LT & ImageNet-LT\\
\midrule
Encoder & ResNet-32 & ResNet-50\\
Encoder dimension & 64 & 2048\\
Projection dimension & 128 & 1024\\
Training epochs & 200 & 90\\
Global batch size & 256 & 256\\
Optimizer & SGD & SGD\\
Initial learning rate & 0.3 & 0.1\\
Momentum & 0.9 & 0.9\\
Weight decay & $2\times10^{-4}$ & $5\times10^{-4}$\\
Training temperature $\tau_0$ & 0.1 & 0.07\\
Contrastive views & two strong & two strong\\
Forward views & two & one weak + two strong\\
Main training prior $\gamma$ & 1 & 1\\
Main intervention strength $\beta$ & 0 & 0\\
Contrastive loss weight & 1 & 1\\
Classification weight (pure / joint) & 0 / 1 & 0 / 1\\
\bottomrule
\end{tabular}
\end{table}

CIFAR uses the recorded cosine schedule.
ImageNet retains its baseline learning-rate schedule.
The main CIFAR imbalance factors are $10,\; 50,\; 100$. Additional controls use $\beta=0.5$, and the ImageNet gain-strength extension additionally uses $\beta=1.5$. \subsection{Frozen-Evaluation Hyperparameters}
\label{sec:frozen-evaluation-hyperparameters}

\begin{table}[H]
\centering
\small
\caption{Frozen and linear-evaluation settings.}
\label{tab:frozen-hyperparameters}
\begin{tabular}{p{0.31\linewidth}p{0.63\linewidth}}
\toprule
Setting & Value or protocol\\
\midrule
Linear-classifier epochs
& 100; fixed final-epoch representation\\

Linear-classifier optimizer
& SGD; batch size 256; learning rate 0.1; momentum 0.9;
zero weight decay; recorded cosine schedule\\

Linear-classifier objective
& Ordinary cross-entropy; no extra prior offset; no vMF correction\\

Main linear fitting pool
& Corresponding long-tail training subset;
ImageNet uses 115,846 fitting images\\

Main Encoder preprocessing
& CIFAR: raw Encoder features;
ImageNet: unit-normalized Encoder features in the main comparison\\

Projection preprocessing
& Original unit-normalized Projection features\\

Historical CIFAR balanced pool
& 50,000 labeled training images; reported separately\\

Frozen vMF base temperature
& $\tau_0=0.1$ for all frozen decision-rule evaluations\\

Main frozen strength
& $\beta=0$; extended controls also use $\beta=0.5$\\

Frozen evaluation prior
& Zero offset in the main frozen table;
frequency offsets are separate controls\\

Frozen vMF fitting pool
& CIFAR/ImageNet: corresponding training subset;
iNaturalist: fixed 353,237-image fitting subset\\

Held-out classification images
& CIFAR: 10,000;
ImageNet: 50,000 test images;
iNaturalist: 24,426 official validation images\\

Boundary calculation
& 2,001 short-arc bracket locations followed by Brent refinement\\
\bottomrule
\end{tabular}
\end{table}

\end{document}